\documentclass{article}

\usepackage[preprint]{neurips_2024}

\usepackage[utf8]{inputenc} 
\usepackage[T1]{fontenc}    
\usepackage{hyperref}       
\usepackage{url}            
\usepackage{booktabs}       
\usepackage{amsfonts}       
\usepackage{nicefrac}       
\usepackage{microtype}      
\usepackage{xcolor}         
\usepackage{graphicx}
\usepackage{subfig}
\usepackage{amsmath}
\usepackage{thmtools}
\usepackage{thm-restate}
\usepackage{dsfont}
\usepackage{algorithm}
\usepackage{algpseudocode}
\usepackage[nameinlink]{cleveref} 
\usepackage{amssymb} 
\usepackage{sectsty}
\usepackage{float} 
\newcounter{rowcounter}

\hypersetup{
	colorlinks=true,
	linkcolor=black, 
	citecolor=black, 
	urlcolor=blue 
}

\newtheorem{theorem}{Theorem}[section]

\newtheorem{definition}{Definition}
\newenvironment{proof}{\noindent\textit{Proof.}}{\hfill$\square$}

\title{Simplification of Risk Averse POMDPs with Performance Guarantees}

\author{%
  Yaacov Pariente\\
  Department of Applied Mathematics\\
  Technion - Israel Institute of Technology\\
  \texttt{yaacovp@campus.technion.ac.il} \\
  \And
  Vadim Indelman\\
  Department of Aerospace Engineering\\
  Technion - Israel Institute of Technology\\
  \texttt{vadim.indelman@technion.ac.il} \\
}

\begin{document}

\maketitle

\begin{abstract}
	Risk averse decision making under uncertainty in partially observable domains is a fundamental problem in AI and essential for reliable autonomous agents. In our case, the problem is modeled using partially observable Markov decision processes (POMDPs), when the value function is the conditional value at risk (CVaR) of the return. Calculating an optimal solution for POMDPs is computationally intractable in general.
	In this work we develop a simplification framework to speedup the evaluation of the value function, while providing performance guarantees. We consider as simplification  a computationally cheaper belief-MDP transition model, that can correspond, e.g., to cheaper observation or transition models. Our contributions include general bounds for CVaR that allow bounding the CVaR of a random variable X, using a random variable Y, by assuming bounds between their cumulative distributions. We then derive bounds for the CVaR value function in a POMDP setting, and show how to bound the value function using the computationally cheaper belief-MDP transition model and	without accessing the computationally expensive model in real-time. Then, we provide theoretical performance guarantees for the estimated bounds. Our results apply for a general simplification of a belief-MDP transition model and support simplification of both the observation and state transition models simultaneously.
\end{abstract}

\section{Introduction}
Autonomous agents gained significant attention across various domains, including healthcare, assistive care, education, industrial scenarios, and decentralized ledger technologies. These agents, ranging from firefighter UAVs to multi-agent trading systems, are designed to operate independently, making decisions and taking actions without direct human intervention. Ensuring that autonomous agents adhere to safe operation in their environments is crucial for real-world deployment. Partial observability characterizes numerous problems, wherein agents lack direct access to the state. In such contexts, ensuring safe autonomous decision-making mandates the adoption of risk-averse POMDP methodologies, incorporating an appropriate risk metric.

Conditional value at risk (CVaR) \citep{rockafellar2000optimization} is a prominent risk measure extensively studied in various fields. A valuable characteristic of CVaR lies in its dual representation \citep{artzner1999coherent}, enabling the interpretation of CVaR as the worst-case expectation of the cost \citep{chow2015risk}. Hence, unlike traditional measures like Value-at-Risk (VaR) that focus solely on the probability of extreme losses, CVaR provides a more comprehensive view by considering the severity of these potential losses. This characteristic makes CVaR particularly useful for decision-makers who are concerned not only with the likelihood of adverse outcomes but also with their potential impact. Furthermore, CVaR is a coherent risk measure, meaning it satisfies certain desirable properties such as sub-additivity and homogeneity, ensuring consistency in risk assessment across different portfolios and time periods \citep{artzner1999coherent}. Available CVaR estimators also have probabilistic guarantees for their deviation from the theoretical CVaR \citep{brown2007large,pmlr-v97-thomas19a}, which enhances its reliability in practice.

Risk can be integrated into planning under uncertainty through various methodologies, such as chance constraints \citep{ono2015chance}, exponential utility functions \citep{koenig1994risk}, distributional robustness \citep{xu2010distributionally} and quantile regression \citep{dabney2018distributional}. Risk measures are functions from a cost random variable to a real number, and should satisfy certain axioms in order to be used in practice \citep{majumdar2020should}. Coherent risk measures satisfy these desirable axioms, thus establishing CVaR as a significant risk metric. General coherent risk measures were used as objectives in POMDPs, constrained MDPs and shortest path problems \citep{ahmadi2021constrained, ahmadi2020risk, ahmadi2021risk, dixit2023risk}, where these results include CVaR as a special case. CVaR was incorporated to MDPs by defining the value function as the CVaR of the return \citep{chow2014algorithms, chow2015risk}, where in \citep{chow2015risk} CVaR MDP was solved using value iteration while providing error guarantees.

Simplification in POMDPs is employed to mitigate computational burden during real-time deployment of POMDP policies, as POMDPs are hard to solve \citep{Papadimitriou1987TheCO}. The term simplification refers to a replacement of any component of a POMDP with a computationally cheaper alternative while providing formal performance guarantees on planning performance. Simplification of the observation model in POMDPs was studied in \citep{levyehudi2023simplifying}, wherein the observation model was replaced with a computationally less expensive alternative, while deriving finite-sample convergence guarantees. \citep{barenboim2024online} considered  simplification of the state and observation spaces, provided deterministic guarantees, and demonstrated how to integrate these guarantees into state-of-the-art solvers. \citep{Zhitnikov22ai} considered simplification of belief-depended rewards and provided deterministic guarantees for Value at Risk.

In this paper we develop a simplification framework with performance guarantees to speed up the evaluation of the value function in risk averse POMDP, when the value function is the CVaR of the return. As far as we know, we are the first to study simplification in this setting. The main contributions of this paper are
\begin{itemize}
	\item Bounds for the CVaR of a random variable X, using a random variable Y, by assuming bounds between the CDFs of the random variables.
	
	\item Lower and upper bounds for the theoretical value function, that utilize the value function derived from a simplified belief-MDP transition model.

	\item Probabilistic guarantees for the deviation between the theoretical value function and its estimated bounds that are computed using the simplified belief-MDP transition model.
\end{itemize}


\section{Preliminaries}
\subsection{Partially Observable Markov Decision Process}\label{sec:POMDP_preliminaries}
A finite horizon POMDP is defined as a tuple $(X,A,Z,T,O,c, b_0)$, where $X,A,Z$ are the state, action and observation spaces respectively. $T(x_{t+1}|x_t,a_t) \triangleq P(x_{t+1}|x_t,a_t)$ is the probability to transition from state $x_t$ to $x_{t+1}$ by taking the action $a_t$. The observation density function $O(z_t|x_t) \triangleq P(z_t|x_t)$ is the probability of observing $z_t$, given the true state is $x_t$. Let $B$ be the set of all beliefs, and define the cost function by $c:B\times A \rightarrow \mathbb{R}$.

The agent views history $H_t \triangleq \{z_{1:t},a_{0:t-1},b_0\}$ and maintains a distribution over the states given the history which is called a belief. The belief $b(x_t) \triangleq P(x_t|H_t)$ for $x_t\in X$, and can be expressed recursively by the following equation $b(x_t) =\eta_tP(z_t|x_t)\int_{x_{t-1}\in X}P(x_t|x_{t-1},a_{t-1})b(x_{t-1})dx_{t-1}$, where $\eta_t$ is a normalization constant. For simplicity, this recursive relation between beliefs would be denoted by $b_t=\psi(b_{t-1}, a_{t-1}, z_t)$, where $b_t\triangleq b(x_t)$.

A policy $a_t=\pi_t(b_t)$ is a mapping from a belief to an action at time t. The cost of action $a_t$ after seeing belief $b_t$ is $c(b_t,a_t) \triangleq E_{x \sim b_t}(c_x(x,a_t))$ such that $|c_x(x,a_t)|\leq R_{max}$. The cost for the finite horizon  $T\in\mathbb{N}$, also known as the return, is $R_{t:T}:=\sum_{\tau=t}^T c(b_\tau,a_\tau)$, which is a measure of the agent’s success at time t.

The value function that is defined with respect to the cost until time horizon T
\begin{equation}\label{Eq:regular_V_function}
	V^\pi (b_k)\triangleq\mathbb{E}[R_{k:T}|b_k,\pi]=\sum_{t=k}^T \mathbb{E}[c(b_t,a_t)|b_k,\pi],
\end{equation}
and the Q function is 
\begin{equation}\label{Eq:regular_Q_function}
	Q^\pi (b_k,a_k) \triangleq \mathbb{E}_{z_{k+1}}[c(b_k,a_k)+V^\pi(b_{k+1})|b_k,a_k].
\end{equation}

A POMDP is a belief-MDP, which is an MDP with belief states. Its transition function can be computed by iterating over observations and states as follows 
\begin{equation}\label{Eq:belief_transition_distribution_}
	\begin{aligned}
		P(b_t|b_{t-1},a_{t-1})=
		\int_{z_{t}\in Z} \!\!\!\!\!\!\!\! P(b_{t}|b_{t-1},a_{t-1},z_{t})
		\int_{x_t\in X} \!\!\!\!\!\!\!\! P(z_t|x_t)
		\int_{x_{t-1}\in X} \!\!\!\!\!\!\!\! \!\!\!\!\!\!P(x_t|a_{t-1},x_{t-1})b_{t-1} dz_t dx_{t-1:t}.
	\end{aligned}
\end{equation}	
For completeness, the proof is available in the Appendix \ref{sec:POMDP_preliminaries}.

\subsection{Conditional Value-at-Risk}\label{sec:cvar_preliminaries}
Let $X$ be a random variable defined on a probability space $(\Omega,F,P)$, such that $E|X|<\infty$ and $F(x) \triangleq P(X\leq x)$. The value at risk at confidence level $\alpha \in (0,1)$ is the $1-\alpha$ quantile of X, i.e $VaR_\alpha (X) \triangleq \sup\{x\in \mathbb{R}:F(x) \leq 1-\alpha\}$. The conditional value at risk (CVaR) at confidence level $\alpha$ is defined as \citep{rockafellar2000optimization}
\begin{equation}
	CVaR_\alpha(X):=\inf_{w\in \mathbb{R}}\{w+\frac{1}{\alpha}\mathbb{E}[(X-w)^+]|w\in \mathbb{R}\},
\end{equation}
where $(x)^+=max(x,0)$. For a smooth $F$, it holds that \citep{pflug2000some}
\begin{equation}\label{def:cvar_as_conditional_expectation}
	CVaR_\alpha(X)=\mathbb{E}[X|X>VaR_\alpha(X)]=\frac{1}{\alpha}\int_{1-\alpha}^1 F^{-1}(v)dv.
\end{equation}
Let $X_i \overset{\text{iid}}{\sim} F$ for $i\in\{1,\dots,n\}$. Denote by
\begin{equation}\label{eq:brown_cvar_estimator}
	\hat{C}_\alpha(X) \triangleq \hat{C}_\alpha(\{X_i\}^{n}_{i=1}) \triangleq \inf_{x\in \mathbb{R}} \Bigl{\{}x+\frac{1}{n\alpha} \sum_{i=1}^n(X_i-x)^+\Bigr{\}}
\end{equation}
the estimate of $CVaR_\alpha(X)$ \citep{brown2007large}. The following two inequalities that
bound the deviation of the estimated CVaR and the true CVaR with high probability, were proved in \citep{brown2007large}. 
\begin{theorem}\label{thm:brown_bounds}
	If $\text{supp}(X) \subseteq [a, b]$ and $X$ has a continuous distribution function, then for any $\delta \in (0, 1]$, 
	\begin{equation}
		P\Bigl{(}CVaR(X)-\hat{C}(X)>(b-a)\sqrt{\frac{5ln(3/\delta)}{\alpha n}}\Bigr{)}\leq \delta,
	\end{equation}
	\begin{equation}
		P\Bigl{(}CVaR(X)-\hat{C}(X)<\frac{(b-a)}{\alpha}\sqrt{\frac{ln(1/\delta)}{2n}}\Bigr{)}\leq \delta.
	\end{equation}
\end{theorem}

\eqref{eq:brown_cvar_estimator} can be expressed as
\begin{equation}\label{eq:cvar_estimator_sorted_sample}
	\hat{C}_\alpha(X)=X^{(n)}-\frac{1}{\alpha}\sum_{i=1}^n (X^{(i)}-X^{(i-1)})\Bigl{(} \frac{i}{n}-(1-\alpha) \Bigr{)}^+,
\end{equation}
where $X^{(i)}$ is the $i$th order statistic of $X_1,\dots,X_n$ in ascending order \citep{pmlr-v97-thomas19a}. More remarks on CVaR can be found in Appendix \ref{sec:cvar_remarks}.

\section{Problem Formulation}\label{sec:problem_formulation}
We consider cases in which the belief-MDP transition model is simplified for computational reasons in a risk-averse setting. Let $\Omega_B$ be the sample space on which the belief random variable is defined, $\Omega=\Omega_B^{T-k+1}, F=2^\Omega$, and define the original and simplified belief probability measures by $P:F\rightarrow [0,1]$ and $P_s:F\rightarrow [0,1]$ respectively. The CDF of $R_{k:T}$ is defined as follows 
\begin{align}
	P(R_{k:T}\leq l|b_k,\pi) = \int_{b_{k+1:T} \in B^{T-k}} P(R_{k:T}\leq l|b_{k:T},\pi)\prod_{i=k+1}^T P(b_i|b_{i-1},\pi)db_{k+1:T}.
\end{align}	 
The simplified CDF of $R_{k:T}$ can be defined using only a simplified belief transition model, as shown in the next theorem.
\begin{theorem}
	\begin{align}
		P_s(R_{k:T}\leq l|b_k,\pi) = \int_{b_{k+1:T} \in B^{T-k}} P(R_{k:T}\leq l|b_{k:T},\pi)\prod_{i=k+1}^T P_s(b_i|b_{i-1},\pi)db_{k+1:T}.
		\label{eq:SimplifiedCDFReturn}
	\end{align}
\end{theorem}
The proof can be found in Appendix \ref{sec:problem_formulation}. In \eqref{Eq:belief_transition_distribution_} we see that the belief-MDP transition model is defined using the state transition and observation models. Hence, simplification of the belief-MDP transition model is general enough to include simplifications of state and observation models. 

In order to optimize a risk averse setting, instead of optimizing the expectation of the return like in \eqref{Eq:regular_V_function} and \eqref{Eq:regular_Q_function}, we define the value function to be the CVaR of the return as follows.
Let $\alpha \in (0,1)$, and define the value and Q functions by
\begin{equation}\label{Eq:cvar_value_function_def}
	\begin{aligned}
		&V^{\pi}_P(b_k,\alpha) \! \triangleq \! CVaR_\alpha^P[\sum_{t=k}^T \! c(b_t, \pi(b_t))|b_k, \pi] \!
		\triangleq \!\! \frac{1}{\alpha} \!\! \int\limits_{1-\alpha}^1 \!\! \sup\{z\in Img(R_{k:T}) \! : \! P(R_{k:T}
		\leq z|b_k,\pi)\leq \tau\}d\tau,
	\end{aligned}
\end{equation}
\begin{equation}\label{Eq:cvar_Q_function_def}
	\begin{aligned}
		Q^{\pi}_P(b_k, a_k, \alpha) &\triangleq CVaR_\alpha^P[c(b_k, a_k) 
		+ \sum_{t=k+1}^T c(b_t, \pi(b_t))|b_k, \pi] \\ 
		&\triangleq \frac{1}{\alpha} \int\limits_{1-\alpha}^1 \sup\{z\in Img(R_{k:T}):
		P(c(b_k,a_k)+R_{k+1:T}\leq z|b_k,a_k, \pi)\leq \tau\}d\tau.
	\end{aligned}
\end{equation}
Our goal is to find lower bound $L_s$ and upper bound $U_s$ that depend only on the simplified CDF of the return \eqref{eq:SimplifiedCDFReturn}. That is, $L_s\leq Q^{\pi}_P(b_k, a_k, \alpha)\leq U_s$. The V and Q functions that are computed with respect to the simplified belief transition model are denoted by $V_{P_s}^\pi (b_k,\alpha)$ and $Q_{P_s}^\pi (b_k,a_k,\alpha)$ respectively.

\section{Bounds for the V and Q functions}
In this section we bound the V and Q functions in \eqref{Eq:cvar_value_function_def} and \eqref{Eq:cvar_Q_function_def}. In section \ref{sec:cvar_bounds} we derive theoretical CVaR bounds that lay the mathematical basis for the derivation of bounds in section \ref{sec:bounds_for_simplified_belief_model}.

\subsection{Theoretical CVaR bounds}\label{sec:cvar_bounds}
Our approach is to derive bounds for $CVaR_\alpha(X)$ using some random variable Y, where we assume that the difference between the CDFs of X and Y is bounded.

\begin{figure}
	\subfloat[]{\includegraphics[width=0.5\textwidth]{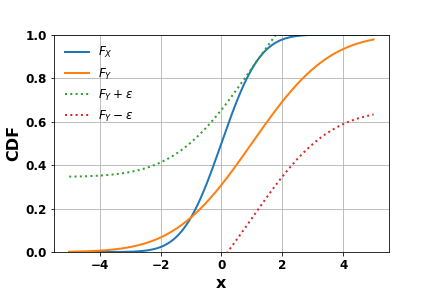}\label{fig:eps_bound}}
	\subfloat[]{\includegraphics[width=0.5\textwidth]{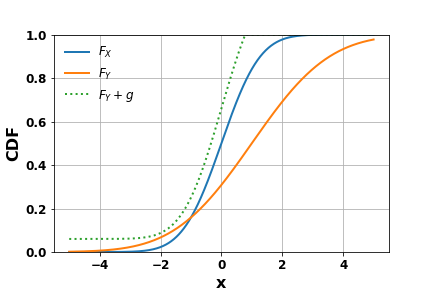}\label{fig:g_bound}}
	\caption{(a) and (b) are illustrations of the bounds on $F_X(x)$ from Theorems \ref{thm:cvar_bound} and \ref{thm:tight_cvar_lower_bound}. Since $g$ from Theorem \ref{thm:tight_cvar_lower_bound} depends on $x$, the bound $F_Y(x)+g(x)$ in (b) is tighter than $F_Y(x)+\epsilon$ in (a).}
\end{figure}
\begin{theorem}\label{thm:cvar_bound}
	Let $X$ and $Y$ be random variables. If there exists $\epsilon \geq 0$ such that $||F_{X} - F_{Y}||_{\infty}\leq \epsilon$, then \begin{enumerate}
		\item If $\epsilon < \alpha$ then $CVaR_\alpha(X)\leq \frac{\alpha-\epsilon}{\alpha}CVaR_{\alpha-\epsilon}(Y) + \frac{\epsilon}{\alpha} \times sup \quad Img(Y)$
		
		\item If $\epsilon \geq \alpha$ then $CVaR_\alpha(X)\leq \sup Img(Y)$
		
		\item If $\epsilon + \alpha < 1$ then $CVaR_\alpha(X)\geq \frac{\alpha+\epsilon}{\alpha}CVaR_{\alpha+\epsilon}(Y)-\frac{\epsilon}{\alpha}CVaR_{\epsilon}(Y)$
		
		\item If $\epsilon + \alpha \geq 1$ then $CVaR_\alpha(X) \geq \frac{1}{\alpha}[(\alpha+\epsilon-1)\inf Img(Y)+ E[Y]-\epsilon CVaR_{\epsilon}(Y)]$
	\end{enumerate}
\end{theorem}
The proof can be found in Appendix \ref{sec:cvar_bounds_proofs}. Theorem \ref{thm:cvar_bound} assumes that the bound is greater than the maximum difference between $F_X$ and $F_Y$ (see Figure \ref{fig:eps_bound}). In Figure \ref{fig:eps_bound} we see that the bound over the CDF (denoted by $\epsilon$ in Theorem \ref{thm:cvar_bound}) is highly conservative around the point $x=-1$. Ideally, the tightest form of bound for $F_X(x)$ is a bound that changes with respect to $x$. That is, $F_X(x)\leq F_Y(x)+g(x)$  for some function $g$ (see Figure \ref{fig:g_bound}).

\begin{theorem}\label{thm:tight_cvar_lower_bound}
	(Tighter CVaR Lower Bound) Let $\alpha \in (0,1)$, $X$ and $Y$ be random variables. Define the a random variable $Y^L$ such that $F_{Y^L}(y) \triangleq min(1, F_Y(y) + g(y))$ for $g:\mathbb{R}\rightarrow [0, \infty)$. Assume 
	$lim_{x \rightarrow -\infty}g(x)=0$, $g$ is continuous from the right and monotonic increasing.
	If $\forall x\in \mathbb{R}, F_X(x)\leq F_Y(x)+g(x)$, then $F_{Y^L}$ is a CDF and $CVaR_\alpha(Y^L)\leq CVaR_\alpha(X).$
\end{theorem}
The proof can be found in Appendix \ref{sec:cvar_bounds}. In Theorem \ref{thm:tight_cvar_lower_bound} we construct a new CDF $F_{Y^L}$, using $F_Y$ and the $g$ function, that stochastically dominates $F_X$ (see Figure \ref{fig:g_bound}). Then, we get the bound because CVaR is a coherent risk measure \citep{artzner1999coherent}. The definition of a coherent risk measure can be found in Appendix \ref{sec:cvar_remarks}. Theorem \ref{thm:tight_cvar_lower_bound} is formulated considering tighter bounds over the CDF than Theorem \ref{thm:cvar_bound}. Therefore the CVaR bounds in Theorem \ref{thm:tight_cvar_lower_bound} are expected to be tighter than those from  Theorem \ref{thm:cvar_bound} for a wise choice of the function $g$. Both bounds for the CVaR in Theorems \ref{thm:cvar_bound} and \ref{thm:tight_cvar_lower_bound} are CVaR of random variables that we can estimate (see \eqref{eq:cvar_estimator_sorted_sample}) and guarantee their performance (see Theorem \ref{thm:brown_bounds}).

Note that $g$ in Theorem \ref{thm:tight_cvar_lower_bound} is assumed to be monotonic increasing and continuous from the right. Assuming that $|F_X(x)-F_Y(x)|\leq g(x)$ for $g:\mathbb{R}\rightarrow [0, \infty)$ that is not necessarily monotonic or continuous, the most raw form of  CVaR bounds is
\begin{equation}\label{eq:general_g_cvar_upper_bound}
	\begin{aligned}
		CVaR_\alpha(X)=\frac{1}{\alpha}\int\limits_{1-\alpha}^1 \inf\{z\in \mathbb{R}:F_X(z)\geq \tau\}d\tau 
		\geq \frac{1}{\alpha}\int\limits_{1-\alpha}^1 \inf\{z\in \mathbb{R}:F_Y(z)+g(z)\geq \tau\}d\tau,
	\end{aligned}
\end{equation}
\begin{equation}\label{eq:general_g_cvar_lower_bound}
	\begin{aligned}
		CVaR_\alpha(X)&=\frac{1}{\alpha}\int\limits_{1-\alpha}^1 \inf\{z\in \mathbb{R}:F_X(z)\geq \tau\}d\tau 
		\leq \frac{1}{\alpha}\int_{1-\alpha}^1 \inf\{z\in \mathbb{R}:F_Y(z)-g(z)\geq \tau\}d\tau.
	\end{aligned}
\end{equation}

Another option is to work directly with the density functions. Let $f_x$ and $f_y$ be the PDFs of X and Y respectively, and $h:\mathbb{R}\rightarrow [0,\infty)$. By assuming $\forall x\in \mathbb{R},|f_x(x)-f_y(x)|\leq h(x)$ and setting $g(z)=\int_{-\infty}^z h(x)dx$, we get the same results as in \eqref{eq:general_g_cvar_upper_bound} and \eqref{eq:general_g_cvar_lower_bound}. 
Theorem \ref{thm:lower_cvar_bound_using_density_bound} shows how to use $g(z)$ in order to achieve a bound that is similar to \eqref{eq:general_g_cvar_lower_bound}, by adding some assumptions on $h$. The advantage of using the bound in Theorem \ref{thm:lower_cvar_bound_using_density_bound} is that the bound is a CVaR of a random variable. Hence, from Theorem \ref{thm:brown_bounds} we can estimate the lower bound and have guarantees for its performance.
\begin{theorem}\label{thm:lower_cvar_bound_using_density_bound}
	Let $\alpha \in (0,1)$, $X$ and $Y$ random variables. Define $h:\mathbb{R}\rightarrow [0,\infty)$ to be a continuous function, $g(z):=\int_{-\infty}^z h(x)dx$ and $Y^L$ to be a random variable such that $F_{Y^L}(y):=\min(1, F_{Y^L}(y) + g(y))$. If $\lim_{z\rightarrow -\infty} g(z)=0$ and $\forall x\in \mathbb{R},f_x(x)\leq f_y(z) + h(x)$, then $F_{Y^L}$ is a CDF and $CVaR_\alpha(Y^L)\leq CVaR_\alpha(X)$.
\end{theorem}
The proof can be found in Appendix \ref{sec:cvar_bounds_proofs}.

\subsection{Theoretical value function bounds}\label{sec:bounds_for_simplified_belief_model}
In this section we derive theoretical bounds for the original value function using the simplified value function. Using the total variation distance (TV distance), the difference between the simplified and original CDFs over the return can be bounded. For convenience, we define the function $f(l, i) \triangleq l -c(b_k,a_k)+ (T-i)R_{max}$ and $1_{R_{k+1:T}\leq l} \triangleq 1_{R_{k+1:i}\leq l}(b_{k+1:i})$ to be an indicator function of the return.
\begin{theorem}\label{thm:simplification_bound_over_belief_distribution}
Let $l\in \mathbb{R}$, then 
	\begin{equation}
	\begin{aligned}
		&|P(R_{k:T} \leq l|b_k,\pi) \! -\! P_s(R_{k:T} \leq l|b_k,\pi)| \! \leq \!\! \! \sum_{i=k+1}^{T-1}  \!\!
		\mathbb{E}_{b_{{k+1}:i} \sim P_s} 
		[1_{R_{k+1:i}\leq f(l,i)}\Delta^s(b_i, a_i)|b_k,\pi],
	\end{aligned}
	\end{equation}
	where $\Delta^s$ is the TV distance that is defined by
	\begin{equation}\label{eq:TVdistBel}
		\Delta^s(b_{t-1},a_{t-1}) \triangleq \int_{b_t\in B} |P(b_t|b_{t-1},a_{t-1})-P_s(b_t|b_{t-1},a_{t-1})|db_t.
	\end{equation}
\end{theorem}
The proof can be found in Appendix \ref{sec:bounds_for_simplified_belief_model_proofs}. The bound in Theorem \ref{thm:simplification_bound_over_belief_distribution} depends on $l$, so it is a point-wise bound as required in Theorem \ref{thm:tight_cvar_lower_bound}. By combining Theorems \ref{thm:simplification_bound_over_belief_distribution} and \ref{thm:tight_cvar_lower_bound} we get a lower bound for the V and Q functions with respect to the simplified model  \eqref{eq:SimplifiedCDFReturn}.
\begin{theorem}\label{thm:tight_lower_bound_for_v_and_q}
	(Tighter Lower Bound for V and Q) Let $\alpha \in (0,1)$, $k,T\in \mathbb{N}$ such that $k<T$, belief $b_k\in B$, action $a_k\in A$ and policy $\pi:X\rightarrow A$. Denote
	\begin{equation}
		\begin{aligned}
			g(l)\triangleq \sum_{i=k+1}^{T-1} \underset{b_{{k+1}:i} \sim P_s}{\mathbb{E}}
			[1_{R_{k+1:i}\leq f(l,i)} 
			\Delta^s(b_i, a_i)|b_k,\pi].
		\end{aligned}
	\end{equation}
	Let $P$ and $P_s$ be two probability measures, and define the 
	random variable $Y^L$ such that 
	\begin{equation}\label{eq:f_y_l_definition}
		F_{Y^L}(y) \triangleq min(1, P_s(R_{k:T}\leq y|b_k,a_k,\pi) + g(y)).
	\end{equation}
	Then,
	\begin{enumerate}
		\item $F_{Y^L}$ is a CDF.
		\item $V^\pi_P(b_k, \alpha)\geq CVaR_\alpha^{P_s}[Y^L|b_k,\pi]$ and $Q^\pi_P(b_k, a_k,\alpha)\geq CVaR_\alpha^{P_s}[Y^L|b_k,a_k,\pi]$
	\end{enumerate}
	
\end{theorem}
The proof can be found in Appendix \ref{sec:bounds_for_simplified_belief_model_proofs}. Under the assumptions of Theorem \ref{thm:simplification_bound_over_belief_distribution} we also get 
\begin{equation}\label{Eq:uniform_bound_for_v_and_q}
	\begin{aligned}
		|P(R_{k:T} \leq l|b_k,\pi)-P_s(R_{k:T} \leq l|b_k,\pi)| 
		&\leq \sum_{i=k+1}^{T-1} \mathbb{E}_{b_{{k+1}:i} \sim P_s}
		[1_{R_{k+1:i}\leq f(l,i)} 
		\Delta^s(b_i, a_i)|b_k,\pi] \\
		&\leq \sum_{i=k+1}^{T-1} \mathbb{E}_{b_{{k+1}:i} \sim P_s}
		[\Delta^s(b_i, a_i)|b_k,\pi],
	\end{aligned}
\end{equation}
which is a uniform bound as required in Theorem \ref{thm:cvar_bound} ($\epsilon$ in the theorem notations). Hence by combining Theorem \ref{thm:cvar_bound} and \eqref{Eq:uniform_bound_for_v_and_q} we get 
\begin{theorem}\label{thm:uniform_lower_and_upper_bounds_for_v_and_q}
	Denote $\epsilon\triangleq\sum_{i=k+1}^{T-1} \mathbb{E}_{b_{{k+1}:i}}[\Delta^s(b_i, a_i)|b_k,\pi]$.
	\begin{enumerate}
		\item \begin{enumerate}
			\item If $\epsilon<\alpha$, then 
			$U_s\triangleq\frac{\alpha-\epsilon}{\alpha}Q^\pi_{P_s}(b_k,a_k,\alpha-\epsilon) + \frac{\epsilon}{\alpha}R_{max}(T-k+1)$
		\item If $\epsilon \geq \alpha$ then $U_s\triangleq(T-k+1)R_{max}$
		\end{enumerate}
		\item \begin{enumerate}
			\item If $\epsilon + \alpha < 1$ then 
			$
			L_s\triangleq\frac{\alpha+\epsilon}{\alpha}Q^\pi_{P_s}(b_k,a_k,\alpha+\epsilon) - \frac{\epsilon}{\alpha} Q^\pi_{P_s}(b_k,a_k, \epsilon)
			$
			\item If $\epsilon + \alpha \geq 1$ then 
			$
			L_s\triangleq\frac{1}{\alpha}[-(\alpha + \epsilon - 1)(T-k+1)R_{max}+Q^\pi_{P_s}(b_k,a_k)-\epsilon Q^\pi_{P_s}(b_k,a_k,\epsilon)]
			$
		\end{enumerate}
	\end{enumerate}
	Then $L_s \leq V_P^\pi(b_k,\alpha)\leq U_s$ and $L_s \leq Q_P^\pi(b_k, a_k,\alpha)\leq U_s$.
\end{theorem}
The proof can be found in Appendix \ref{sec:bounds_for_simplified_belief_model_proofs}.
Theorems \ref{thm:uniform_lower_and_upper_bounds_for_v_and_q} and \ref{thm:tight_lower_bound_for_v_and_q} rely on the bounds from Theorems \ref{thm:tight_cvar_lower_bound} and \ref{thm:cvar_bound} respectively. The lower bound in Theorem \ref{thm:tight_cvar_lower_bound} is expected to be tighter than the one in Theorem \ref{thm:cvar_bound} (see section \ref{sec:cvar_bounds}) and therefore we expect the lower bound in Theorem \ref{thm:tight_lower_bound_for_v_and_q} to be tighter than the one in Theorem \ref{thm:uniform_lower_and_upper_bounds_for_v_and_q}. We can also use a mix of the lower bound from Theorem \ref{thm:tight_lower_bound_for_v_and_q} and the upper bound from Theorem \ref{thm:uniform_lower_and_upper_bounds_for_v_and_q}.

\section{Bound estimation}
\subsection{Q function estimator}\label{sec:q_func_estimator}
Let $M_P$ denote a Particle Belief MDP (PB-MDP) with respect to our POMDP, that is as defined in \citep{lim2023optimality}. Algorithm \ref{alg:qv_estimation} gives estimates $\hat{Q}^\pi_{M_P}(\bar{b}, a_k,\alpha),\hat{Q}^\pi_{M_{P_s}}(\bar{b}, a_k,\alpha)$ for $Q_{M_P}^\pi(\bar{b}, a_k,\alpha),Q_{M_{P_s}}^\pi(\bar{b}, a_k,\alpha)$ respectively, where the policy $\pi$ is given. Let $\bar{b}_d=\{(x^j_d,w_d^j)\}_{j=1}^{N_x}$ be a particle-based belief representation, and $\{\bar{b}_d^i\}$ a sequence of $C$ particle belief samples with $i\in \{1,\dots C\}$ and depth $d\in \{k,\dots T\}$ which is generated according to Algorithm \ref{alg:qv_estimation}. Denote the estimated return that corresponds to the $i$th belief sequence by 
\begin{equation}\label{def:particle_return}
	\bar{R}^i_{k:T} \triangleq \sum_{t=k}^T r(\bar{b}_t^i, \pi(\bar{b}_t^i)) \triangleq \sum_{t=k}^T\frac{1}{\sum_{j=1}^C w_t^j}\sum_{i=1}^C r(x_t^i,a_t)w_t^i,
\end{equation}
and the estimated CVaR by (\ref{eq:cvar_estimator_sorted_sample})
\begin{equation}
	\hat{C}(\{R^i\}_{i=1}^C)=R^{(C)}-\frac{1}{\alpha}\sum_{i=1}^C (R^{(i)}-R^{(i-1)})\Bigl(\frac{i}{n}-(1-\alpha) \Bigr)^+.
\end{equation}
From \citep{pmlr-v97-thomas19a} we know that 
$\hat{C}_\alpha(\{R^i\}_{i=1}^C)=\inf_x\{x:x + \frac{1}{C\alpha}\sum_{i=1}^C (R^i-x)^+\}.$
Define the Q function that is estimated using Algorithm \ref{alg:qv_estimation} by 
$
\hat{Q}^{\pi}_{M_P}(\bar{b},a,\alpha)\triangleq\hat{C}_\alpha(\{\bar{R}^i_{k:T}\}_{i=1}^C).
$ 

\subsection{CDF bound estimation}\label{sec:cdf_bound_estimation}
In this section we develop estimators for $\epsilon$ and $g(l)$ from Theorems \ref{thm:uniform_lower_and_upper_bounds_for_v_and_q} and \ref{thm:tight_lower_bound_for_v_and_q} respectively, which will be used in the next section. The bounds in Theorems \ref{thm:tight_lower_bound_for_v_and_q} and \ref{thm:uniform_lower_and_upper_bounds_for_v_and_q} are  defined using the original belief-MDP transition model, which could be expensive in real-time. We follow the methodology of \citep{levyehudi2023simplifying} and compute the bound without accessing the original belief-MDP transition model in real-time. We sample $\{b_n^{\Delta}\}_{n=1}^{N_\Delta} \sim Q_0(b)$, named delta beliefs, and evaluate $\Delta^s(b_n^\Delta)$ from \eqref{eq:TVdistBel} for $n=1,\dots,N_\Delta$. In the current setting $Q_0$ represents a general distribution over beliefs. In particular, it could correspond to a distribution over particle beliefs. During online planning, we reweight $\Delta^s$ with importance sampling for belief sample $\bar{b}^i_d$. Denote 
\begin{align}
	m_i\triangleq E_{b_{k+1:i} \sim P_s}[\Delta^s(b_i,a_i)|\bar{b}_k,\pi]=E_{b_i\sim Q_0}[\frac{P_s(b_{i}|\bar{b}_{k},\pi)}{Q_0(b_i)}\Delta^s(b_i,a_i)].
\end{align}
The estimator of $m_i$ is defined as follows
\begin{equation}\label{eq:m_i_estimator}
	\hat{m}_i\triangleq \frac{1}{N_\Delta}\sum_{n=1}^{N_\Delta}\frac{P_s(b_i^{\Delta, n}|\bar{b}_{k},\pi)}{Q_0(b_i^{\Delta,n})}\hat{\Delta}^s(b_i^{\Delta,n},\pi(b_i^{\Delta,n})),
\end{equation}

where $\hat{\Delta}$ is some estimator of $\Delta$. Denote $\epsilon\triangleq \sum_{i=j+1}^{T-1} m_i$ to be the $\epsilon$ from Theorem \ref{thm:uniform_lower_and_upper_bounds_for_v_and_q}, which will be estimated by $\hat{\epsilon}\triangleq \sum_{i=j+1}^{T-1} \hat{m}_i$. 

\begin{theorem}\label{thm:m_i_sum_bound_guarantees}
	Let $v>0,\delta\in (0,1)$. If $\hat{\Delta}$ is unbiased, then for $N_\Delta\geq -8B^2 \frac{ln(\frac{\delta}{4(T-1-k)})}{(v/(T-1-k))^2}$ it holds that $P(|\hat{\epsilon} - \epsilon|>2v)\leq \delta$ for $B=\max_{i\in \{k+1,\dots,T-1\}} B_i$ where $B_i=\sup_{b_i}P_s(b_i|\bar{b}_k)/Q_0(b_i)$.
\end{theorem}
The proof can be found in Appendix \ref{sec:cdf_bound_estimation_proofs}. Similarly to the analysis above, we want to estimate the $g$ function from Theorem \ref{thm:tight_lower_bound_for_v_and_q}. Denote 
\begin{equation}
	\begin{aligned}
		g_i(l) &\triangleq \mathbb{E}_{b_{{k+1}:i} \sim P_s}
		[1_{\bar{R}_{k+1:i}\leq f(l,i)}\Delta^s(b_i, a_i)|\bar{b}_k,\pi] 
		=\mathbb{E}_{b_i \sim Q_0}
		[\frac{P_s(b_{i}|\bar{b}_{k},\pi)}{Q_0(b_i)}1_{\bar{R}_{k+1:i}\leq f(l,i)}\Delta^s(b_i, a_i)].
	\end{aligned}
\end{equation}
The estimator for $g_i$ is defined as follows 
\begin{equation}
	\hat{g}_i(l) \triangleq \frac{1}{N_\Delta}\sum_{n=1}^{N_\Delta}\frac{P_s(b_i^{\Delta,n}|\bar{b}_{k},\pi)}{Q_0(b_i^{\Delta,n})}\hat{\Delta}^s(b_i^{\Delta,n},\pi(b_i^{\Delta,n}))1_{\bar{R}^n_{k+1:i}\leq f(l,i)},
\end{equation}

where $\hat{\Delta}^s$ is some estimator of $\Delta^s$. Denote by $g(l) \triangleq \sum_{i=k+1}^{T-1} g_i(l)$ to be the function $g$ from Theorem \ref{thm:tight_lower_bound_for_v_and_q}, which will be estimated by $\hat{g}(l) \triangleq \sum_{i=k+1}^{T-1} \hat{g}_i(l)$.
\begin{theorem}\label{thm:g_estimator_convergence_guarantees}
	Let $v>0,\delta \in (0,1),l\in \mathbb{R}$ and denote $B_i \triangleq \sup_{b_i} \frac{P(b_i|\bar{b}_k, \pi)}{Q_0{b_i}}, B \triangleq \max_{i\in \{k+1,\dots,T-1\}} B_i$. If $N_\Delta \geq -ln(\frac{\delta/(T-1-k)}{2})\frac{2B^2}{v^2 / (T-1-k)^2}$ and $\hat{\Delta}$ is unbiased then 
	\begin{equation}
		P(|g(l) - \hat{g}(l)|>v)\leq \delta.
	\end{equation}
\end{theorem}
The proof can be found in Appendix \ref{sec:cdf_bound_estimation_proofs}. 
We compute $\hat{g}(l)$ on bins to enable the user to control the number of times $\hat{g}$ is computed. This approach is beneficial when $\hat{g}$ is computationally demanding and the user prioritizes speed over accuracy. Formally, we define
\begin{equation}\label{eq:h_estimators}
	h^+(l) \triangleq \sum_{i=1}^I g(a_i)1_{l\in (k_{i-1},k_i]}, h^-(l) \triangleq \sum_{i=1}^I g(a_{i-1})1_{l\in (k_{i-1},k_i]},
\end{equation}
to be the upper and lower bounds for $g$ where $(k_{i-1}, k_i]$ are arbitrarily chosen bins. Also define their estimators by 
\begin{equation}\label{eq:h_functions}
	\hat{h}^+(l) \triangleq \sum_{i=1}^I \hat{g}(a_i)1_{l\in (k_{i-1},k_i]}, \hat{h}^-(l) \triangleq \sum_{i=1}^I \hat{g}(a_{i-1})1_{l\in (k_{i-1},k_i]},
\end{equation} 
\begin{theorem}\label{thm:h_guarantees}
	Let $\alpha,\delta\in (0,1), v>0$. If $N_\Delta \geq -ln(\frac{(\delta/I)/(T-1-k)}{2})\frac{2B^2}{v^2 / (T-1-k)^2}$ then
	\begin{equation}
		P(\sup_{l\in \mathbb{R}}\{g(l)-\hat{h}^+(l)\}>v|\bar{b}_k,a_k,\pi)\leq \delta
	\end{equation}
	\begin{equation}
		P(\sup_{l\in \mathbb{R}}\{\hat{h}^-(l)-g(l)\}>v|\bar{b}_k,a_k,\pi)\leq \delta
	\end{equation}
\end{theorem}
The proof can be found in Appendix \ref{sec:cdf_bound_estimation_proofs}. 

\subsection{Performance guarantees}\label{sec:performance_guarantees}
In section \ref{sec:bounds_for_simplified_belief_model} we proved theoretical bounds for V and Q, and in sections \ref{sec:q_func_estimator} and \ref{sec:cdf_bound_estimation} we saw how to estimate these theoretical bounds. In this section we provide finite sample guarantees for the deviation between $Q^\pi_{M_P}$ and $\hat{Q}^\pi_{M_{P_s}}$. First, we start with guarantees for Theorem \ref{thm:uniform_lower_and_upper_bounds_for_v_and_q}. 
\begin{theorem}\label{thm:unfirom_bound_convergence_guarantees_with_estimated_epsilon}(Bound guarantees)
	Let $\delta \in (0,1),\alpha\in (0,1),v>0$. Denote \begin{enumerate}
		\item $L_1 \triangleq \frac{\alpha+\hat{\epsilon}-4v}{\alpha}\hat{Q}^\pi_{M_{P_s}}(\bar{b}_k,a_k,\alpha+\hat{\epsilon}) - \frac{\hat{\epsilon}}{\alpha}\hat{Q}^\pi_{M_{P_s}}(\bar{b}_k,a_k,\hat{\epsilon}-4v)$
		
		\item $L_2 \triangleq \frac{1}{\alpha}[\hat{Q}^\pi_{M_{P_s}}(\bar{b}_k,a_k)-(\hat{\epsilon}+4v)\hat{Q}^\pi_{M_{P_s}}(\bar{b}_k,a_k,\alpha)-(\alpha+\hat{\epsilon}+4v-1)(T-k+1)R_{max}]$
		
		\item $U \triangleq \frac{\alpha-\hat{\epsilon}+4v}{\alpha}\hat{Q}^\pi_{M_{P_s}}(\bar{b}_k,a_k,\alpha-\hat{\epsilon}) + \frac{\hat{\epsilon}}{\alpha}R_{max}(T-k+1)$
	\end{enumerate} 
	If $N_\Delta\geq -8B^2 \frac{ln(\frac{\delta/2}{4(T-k)})}{(v/(T-k))^2}$ then the following hold \begin{enumerate}
		\item If $\hat{\epsilon}+\alpha<1$ then $
			P(L_1-Q_{M_P}^\pi(\bar{b}_k,a_k,\alpha)>\lambda_1+\lambda_2)\leq \delta
		$
		for $\lambda_1=-\frac{2R_{max}(T-k+1)}{\alpha}\sqrt{\frac{ln(1/(\delta/4))}{2C}}$ and $\lambda_2=\frac{\sqrt{\hat{\epsilon}}}{\alpha}2R_{max}(T-k+1)\sqrt{\frac{5ln(3/(\delta/4))}{C}}$.
		
		\item If $\hat{\epsilon}+\alpha\geq 1$ then $
		P(L_2-Q_{M_P}^\pi(\bar{b}_k,a_k,\alpha)>\eta_1+\eta_2)\leq \delta
		$ for $\eta_1 \triangleq \sqrt{-\frac{ln(\delta / 4)R_{max}(T-k+1)}{C^2\alpha^2}}$ and $\eta_2 \triangleq \frac{2\sqrt{\hat{\epsilon}+4v}}{\alpha}R_{max}(T-k+1)\sqrt{\frac{5ln(3/(\delta/4))}{C}}$
		
		\item If $\alpha>\hat{\epsilon}$ then $
		P(Q^\pi_{M_{P}}(\bar{b}_k,a_k,\alpha) - U> \lambda)\leq \delta
		$ for $\lambda=2R_{max}(T-k+1)\frac{\sqrt{\alpha-\hat{\epsilon}}}{\alpha}\sqrt{\frac{5ln(3/(\delta/2))}{C}}$
	\end{enumerate}
\end{theorem}
The proof can be found in Appendix \ref{sec:performance_guarantees_proofs}. Now we estimate the lower bound from Theorem \ref{thm:tight_lower_bound_for_v_and_q} and guarantee its performance. We replace the theoretical $g$ in Theorem \ref{thm:tight_lower_bound_for_v_and_q} with its upper bound estimator $\hat{h}^+(l)$ plus some constant, and the theoretical return $R_{k:T}$ with a particle belief return $\bar{R}_{k:T}$. Formally, we construct $$\hat{F}_{Y^L}(l) \triangleq min(1, P_{M_{P_s}}(\bar{R}_{k:T}\leq l|\bar{b}_k,a_k,\pi) + \hat{h}^+(l) + \eta)$$ for $\eta>0$, as an estimator for $F_{Y^L}$ in Theorem \ref{thm:tight_lower_bound_for_v_and_q}. Let $\hat{Q}^{\pi,\hat{h}^+ +\eta}_{M_{P_s}}(\bar{b}_k,a_k,\alpha) \triangleq \hat{C}_\alpha(\{\bar{R}^{Y^L}_{i,\hat{h}^+ +\eta}\}_{i=1}^{N_\Delta})$ be a lower bound estimator for $Q^\pi_{M_{P_s}}(\bar{b}_k,a_k,\alpha)$, where $\bar{R}^{Y^L}_{i,\hat{h}^+ +\eta} \overset{iid}{\sim} \hat{F}_{Y^L}$. By integrating Theorem \ref{thm:tight_lower_bound_for_v_and_q} with the guarantees from Theorem \ref{thm:h_guarantees}, we get our final result of finite sample guarantees for the deviation between the estimated simplification-based lower bound $\hat{Q}^{\pi,g+v}_{M_{P_s}}(\bar{b}_k,a_k,\alpha)$ and $Q_{M_P}^\pi(\bar{b}_k,a_k,\alpha)$.
\begin{theorem}\label{thm:tight_lower_bound_guarantees_estimated_g}
	Let $\eta>0,\delta \in (0,1),\alpha \in (0,1)$. 
	If $N_\Delta \geq -ln(\frac{((\delta/4)/I)/(T-1-k)}{2})\frac{2B^2}{\eta^2 / (T-1-k)^2}$, then 
	\begin{align}
		P(\hat{Q}_{M_{P_s}}^{\pi,\hat{h}^+ +
			\eta}(\bar{b}_k,a_k,\alpha) - Q_{M_P}^\pi(\bar{b}_k,a_k,\alpha)>v|\bar{b}_k,\pi)\leq \delta,
	\end{align}
	for $v=\frac{2R_{max}(T-k+1)}{\alpha}\sqrt{\frac{ln(1/(\delta/4))}{2N_\Delta}}$ and $B_i \triangleq \sup_{b_i} \frac{P(b_i|\bar{b}_k, \pi)}{Q_0{b_i}}, B \triangleq \max_{i\in \{k+1,\dots,T-1\}} B_i$.
\end{theorem}
The proof can be found in Appendix \ref{sec:performance_guarantees_proofs}.

\section{Conclusions}
We presented a framework for simplification of risk-averse POMDPs with performance guarantees, that allows reducing policy evaluation time in online deployment considering CVaR as the value function. Specifically, we studied the effects of replacing a computationally expensive belief-MDP transition model with a computationally cheaper one. To that end, we first established bounds for the CVaR of a random variable X using another random variable Y, by assuming bounds over their CDFs and PDFs difference. These bounds, which are of independent interest, were then used to bound with high probability the difference  between the value functions  that utilize the original and simplified belief-MDP transition models. A limitation of this study lies in the computation of \eqref{eq:m_i_estimator}, given that the computation of belief-MDP transition models remains an unresolved challenge within the field \citep{lim2023optimality}. 

\section{Acknowledgments}

This work was partially supported by the Israel Science Foundation (ISF).

\bibliographystyle{plainnat}
\bibliography{references.bib}

\begin{thebibliography}{21}
\providecommand{\natexlab}[1]{#1}
\providecommand{\url}[1]{\texttt{#1}}
\expandafter\ifx\csname urlstyle\endcsname\relax
  \providecommand{\doi}[1]{doi: #1}\else
  \providecommand{\doi}{doi: \begingroup \urlstyle{rm}\Url}\fi

\bibitem[Ahmadi et~al.(2020)Ahmadi, Ono, Ingham, Murray, and
  Ames]{ahmadi2020risk}
Mohamadreza Ahmadi, Masahiro Ono, Michel~D Ingham, Richard~M Murray, and
  Aaron~D Ames.
\newblock Risk-averse planning under uncertainty.
\newblock In \emph{2020 American Control Conference (ACC)}, pages 3305--3312.
  IEEE, 2020.

\bibitem[Ahmadi et~al.(2021{\natexlab{a}})Ahmadi, Dixit, Burdick, and
  Ames]{ahmadi2021risk}
Mohamadreza Ahmadi, Anushri Dixit, Joel~W Burdick, and Aaron~D Ames.
\newblock Risk-averse stochastic shortest path planning.
\newblock In \emph{2021 60th IEEE Conference on Decision and Control (CDC)},
  pages 5199--5204. IEEE, 2021{\natexlab{a}}.

\bibitem[Ahmadi et~al.(2021{\natexlab{b}})Ahmadi, Rosolia, Ingham, Murray, and
  Ames]{ahmadi2021constrained}
Mohamadreza Ahmadi, Ugo Rosolia, Michel~D Ingham, Richard~M Murray, and Aaron~D
  Ames.
\newblock Constrained risk-averse markov decision processes.
\newblock In \emph{Proceedings of the AAAI Conference on Artificial
  Intelligence}, volume~35, pages 11718--11725, 2021{\natexlab{b}}.

\bibitem[Artzner et~al.(1999)Artzner, Delbaen, Eber, and
  Heath]{artzner1999coherent}
Philippe Artzner, Freddy Delbaen, Jean-Marc Eber, and David Heath.
\newblock Coherent measures of risk.
\newblock \emph{Mathematical finance}, 9\penalty0 (3):\penalty0 203--228, 1999.

\bibitem[Barenboim and Indelman(2024)]{barenboim2024online}
Moran Barenboim and Vadim Indelman.
\newblock Online pomdp planning with anytime deterministic guarantees.
\newblock \emph{Advances in Neural Information Processing Systems}, 36, 2024.

\bibitem[Brown(2007)]{brown2007large}
David~B Brown.
\newblock Large deviations bounds for estimating conditional value-at-risk.
\newblock \emph{Operations Research Letters}, 35\penalty0 (6):\penalty0
  722--730, 2007.

\bibitem[Chow and Ghavamzadeh(2014)]{chow2014algorithms}
Yinlam Chow and Mohammad Ghavamzadeh.
\newblock Algorithms for cvar optimization in mdps.
\newblock \emph{Advances in neural information processing systems}, 27, 2014.

\bibitem[Chow et~al.(2015)Chow, Tamar, Mannor, and Pavone]{chow2015risk}
Yinlam Chow, Aviv Tamar, Shie Mannor, and Marco Pavone.
\newblock Risk-sensitive and robust decision-making: a cvar optimization
  approach.
\newblock \emph{Advances in neural information processing systems}, 28, 2015.

\bibitem[Dabney et~al.(2018)Dabney, Rowland, Bellemare, and
  Munos]{dabney2018distributional}
Will Dabney, Mark Rowland, Marc Bellemare, and R{\'e}mi Munos.
\newblock Distributional reinforcement learning with quantile regression.
\newblock In \emph{Proceedings of the AAAI conference on artificial
  intelligence}, volume~32, 2018.

\bibitem[Dixit et~al.(2023)Dixit, Ahmadi, and Burdick]{dixit2023risk}
Anushri Dixit, Mohamadreza Ahmadi, and Joel~W Burdick.
\newblock Risk-averse receding horizon motion planning for obstacle avoidance
  using coherent risk measures.
\newblock \emph{Artificial Intelligence}, 325:\penalty0 104018, 2023.

\bibitem[Koenig and Simmons(1994)]{koenig1994risk}
Sven Koenig and Reid~G Simmons.
\newblock Risk-sensitive planning with probabilistic decision graphs.
\newblock In \emph{Principles of Knowledge Representation and Reasoning}, pages
  363--373. Elsevier, 1994.

\bibitem[Lev-Yehudi et~al.(2023)Lev-Yehudi, Barenboim, and
  Indelman]{levyehudi2023simplifying}
Idan Lev-Yehudi, Moran Barenboim, and Vadim Indelman.
\newblock Simplifying complex observation models in continuous pomdp planning
  with probabilistic guarantees and practice, 2023.

\bibitem[Lim et~al.(2023)Lim, Becker, Kochenderfer, Tomlin, and
  Sunberg]{lim2023optimality}
Michael~H Lim, Tyler~J Becker, Mykel~J Kochenderfer, Claire~J Tomlin, and
  Zachary~N Sunberg.
\newblock Optimality guarantees for particle belief approximation of pomdps.
\newblock \emph{Journal of Artificial Intelligence Research}, 77:\penalty0
  1591--1636, 2023.

\bibitem[Majumdar and Pavone(2020)]{majumdar2020should}
Anirudha Majumdar and Marco Pavone.
\newblock How should a robot assess risk? towards an axiomatic theory of risk
  in robotics.
\newblock In \emph{Robotics Research: The 18th International Symposium ISRR},
  pages 75--84. Springer, 2020.

\bibitem[Ono et~al.(2015)Ono, Pavone, Kuwata, and Balaram]{ono2015chance}
Masahiro Ono, Marco Pavone, Yoshiaki Kuwata, and J~Balaram.
\newblock Chance-constrained dynamic programming with application to risk-aware
  robotic space exploration.
\newblock \emph{Autonomous Robots}, 39:\penalty0 555--571, 2015.

\bibitem[Papadimitriou and Tsitsiklis(1987)]{Papadimitriou1987TheCO}
Christos~H. Papadimitriou and John~N. Tsitsiklis.
\newblock The complexity of markov decision processes.
\newblock \emph{Math. Oper. Res.}, 12:\penalty0 441--450, 1987.
\newblock URL \url{https://api.semanticscholar.org/CorpusID:29322444}.

\bibitem[Pflug(2000)]{pflug2000some}
Georg~Ch Pflug.
\newblock Some remarks on the value-at-risk and the conditional value-at-risk.
\newblock \emph{Probabilistic constrained optimization: Methodology and
  applications}, pages 272--281, 2000.

\bibitem[Rockafellar et~al.(2000)Rockafellar, Uryasev,
  et~al.]{rockafellar2000optimization}
R~Tyrrell Rockafellar, Stanislav Uryasev, et~al.
\newblock Optimization of conditional value-at-risk.
\newblock \emph{Journal of risk}, 2:\penalty0 21--42, 2000.

\bibitem[Thomas and Learned-Miller(2019)]{pmlr-v97-thomas19a}
Philip Thomas and Erik Learned-Miller.
\newblock Concentration inequalities for conditional value at risk.
\newblock In Kamalika Chaudhuri and Ruslan Salakhutdinov, editors,
  \emph{Proceedings of the 36th International Conference on Machine Learning},
  volume~97 of \emph{Proceedings of Machine Learning Research}, pages
  6225--6233. PMLR, 09--15 Jun 2019.
\newblock URL \url{https://proceedings.mlr.press/v97/thomas19a.html}.

\bibitem[Xu and Mannor(2010)]{xu2010distributionally}
Huan Xu and Shie Mannor.
\newblock Distributionally robust markov decision processes.
\newblock \emph{Advances in Neural Information Processing Systems}, 23, 2010.

\bibitem[Zhitnikov and Indelman(2022)]{Zhitnikov22ai}
A.~Zhitnikov and V.~Indelman.
\newblock Simplified risk aware decision making with belief dependent rewards
  in partially observable domains.
\newblock \emph{Artificial Intelligence, Special Issue on ``Risk-Aware
  Autonomous Systems: Theory and Practice"}, 2022.

\end{thebibliography}

\appendix

\section{Appendix}

\subsection{Conditional value at risk (CVaR) remarks}\label{sec:cvar_remarks}

\begin{definition} (First order stochastic dominance)
	Let $X_1,X_2$ be random variables. We say that $X_1$ stochastically dominates in order 1 $X_2$ and denote $X_1\geq_{SD(1)} X_2$, iff
	$\mathbb{E}[\psi(X_1)]\geq \mathbb{E}[\psi(X_2)]$ for all integrable monotonic functions $\psi$.
\end{definition}

An equivalent definition for \ref{def:cvar_as_conditional_expectation} is that $X\geq Y$ iff $F_X(z)\leq F_Y(z)$.
In this paper, the sign $\leq$ between random variables denotes first order stochastic dominance.

\eqref{def:cvar_as_conditional_expectation} leads to other representations of CVaR \citep{pflug2000some}
\begin{equation}\label{eq:pflug_cvar_form}
	\begin{aligned}
		CVaR_\alpha(X)&=\mathbb{E}[X|X>F^{-1}(1-\alpha)]=\frac{1}{\alpha}\int_{1-\alpha}^1 F^{-1}(v)dv
		=\frac{1}{\alpha}\int_{(F^{-1}(1-\alpha), \infty)}u dF(u).
	\end{aligned}
\end{equation}

CVaR is a coherent risk measure \citep{pflug2000some} that satisfies
\begin{enumerate}
	\item Translation equivariant: $$CVaR_\alpha(X+c)=CVaR_\alpha(X)+c.$$
	\item Positively homogeneous: $CVaR_\alpha(cX)=cCVaR(X)$ for $c>0$.
	\item Convexity: For random variables $X_1,X_2$ and $\lambda\in (0,1)$, 
	$$CVaR_\alpha(\lambda X_1+(1-\lambda)X_2)\leq \lambda CVaR_\alpha(X_1) + (1-\lambda)CVaR_\alpha(X_2)$$
	\item Monotonicity: if $X_1\leq_{SD(1)} X_2$, then $CVaR_\alpha(X_1)\leq CVaR_\alpha(X_2)$.
\end{enumerate}

\subsection{Proofs of section \ref{sec:POMDP_preliminaries}}
\begin{theorem}\label{thm:belief_transition_probability}
	$$\begin{aligned}
		P(b_t&|b_{t-1},a_{t-1})= \\
		&\int_{z_{t}\in Z}P(b_{t}|b_{t-1},a_{t-1},z_{t})
		\int_{x_t\in X}P(z_t|x_t)
		\int_{x_{t-1}\in X}P(x_t|a_{t-1},x_{t-1})b(x_{t-1}) dz_t dx_{t-1}dx_t
	\end{aligned}$$	
\end{theorem}

\begin{proof}
	\begin{equation}
		\begin{aligned}
			&P(b_t|b_{t-1},a_{t-1}) = \int_{z_t\in Z}P(b_{t}|b_{t-1},a_{t-1},z_{t})P(z_{t}|b_{t-1},a_{t-1})dz_t\\
			& = \int_{z_{t}\in Z}P(b_{t}|b_{t-1},a_{t-1},z_{t})
			\int_{x_t\in X}P(z_t|x_t)P(x_t|b_{t-1},a_{t-1})dz_t dx_t\\
			& = \int_{z_{t}\in Z}P(b_{t}|b_{t-1},a_{t-1},z_{t})\int_{x_t\in X}P(z_t|x_t)
			\int_{x_{t-1}\in X}P(x_t|a_{t-1},x_{t-1})P(x_{t-1}|b_{t-1},a_{t-1})dz_t dx_{t-1}dx_t\\
			& = \int_{z_{t}\in Z}P(b_{t}|b_{t-1},a_{t-1},z_{t})
			\int_{x_t\in X}P(z_t|x_t)
			\int_{x_{t-1}\in X}P(x_t|a_{t-1},x_{t-1})b(x_{t-1}) dz_t dx_{t-1}dx_t
		\end{aligned}
	\end{equation}
\end{proof}

\subsection{Proofs of section \ref{sec:problem_formulation}}
\begin{theorem}\label{prf:simplified_cdf_expression}
	$$P_s(R_{k:T}\leq l|b_k,\pi) = \int_{b_{k+1:T} \in B^{T-k}} P(R_{k:T}\leq l|b_{k:T},\pi)\prod_{i=k+1}^T P_s(b_i|b_{i-1},\pi)db_{k+1:T}$$
\end{theorem}
\begin{proof}
	$R_{k:T}$ is a function of the beliefs $b_k,\dots,b_T$, and therefore integrating $R_{k:T}$ with respect to the probability measure $P_s$ is well defined.
	$$\begin{aligned}
		&P_s(R_{k:T}\leq l|b_k,\pi)=\int_{b_{k+1:T} \in B^{T-k}} P_s(R_{k:T}\leq l|b_{k:T},\pi)P_s(b_{k+1:T}|b_k,\pi)db_{k+1:T} \\
		&\overset{1}{=}\int_{b_{k+1:T} \in B^{T-k}} P_s(R_{k:T}\leq l|b_{k:T},\pi)\prod_{i=k+1}^T P_s(b_{i}|b_{i-1},\pi)db_{k+1:T} \\
		&\overset{2}{=}\int_{b_{k+1:T} \in B^{T-k}} 1_{R_{k:T}\leq l}\prod_{i=k+1}^T P_s(b_{i}|b_{i-1},\pi)db_{k+1:T}\\
		&\overset{2}{=}\int_{b_{k+1:T} \in B^{T-k}} P(R_{k:T}\leq l|b_{k:T}, \pi)\prod_{i=k+1}^T  P_s(b_{i}|b_{i-1},\pi)db_{k+1:T}
	\end{aligned}$$
	$^1$ Belief-MDP transition model is Markovian. \\
	$^2$ $R_{k:T}$ is a constant with respect constant beliefs $b_k,\dots,b_T$. 
\end{proof}

\subsection{Proofs of section \ref{sec:cvar_bounds}}\label{sec:cvar_bounds_proofs}

\begin{theorem}\label{prf:cvar_bound}
	Let $X$ and $Y$ be random variables. If there exists $\epsilon \geq 0$ such that $||F_{X} - F_{Y}||_{\infty}\leq \epsilon$, then \begin{enumerate}
		\item If $\epsilon < \alpha$ then $CVaR_\alpha(X)\leq \frac{\alpha-\epsilon}{\alpha}CVaR_{\alpha-\epsilon}(Y) + \frac{\epsilon}{\alpha} \times sup \quad Img(Y)$
		
		\item If $\epsilon \geq \alpha$ then $CVaR_\alpha(X)\leq \sup Img(Y)$
		
		\item If $\epsilon + \alpha < 1$ then $CVaR_\alpha(X)\geq \frac{\alpha+\epsilon}{\alpha}CVaR_{\alpha+\epsilon}(Y)-\frac{\epsilon}{\alpha}CVaR_{\epsilon}(Y)$
		
		\item If $\epsilon + \alpha \geq 1$ then $CVaR_\alpha(X) \geq \frac{1}{\alpha}[(\alpha+\epsilon-1)\inf Img(Y)+ E[Y]-\epsilon CVaR_{\epsilon}(Y)]$
	\end{enumerate}
\end{theorem}
\begin{proof}
	\begin{equation}
		\begin{aligned}
			CVaR_\alpha(X)&=\frac{1}{\alpha}\int_{1-\alpha}^1 sup\{z|F_X(z)\leq \tau\} d\tau\\
			&\leq \frac{1}{\alpha}\int_{1-\alpha}^1 sup\{z|F_Y(z)\leq \tau + \epsilon\} d\tau\\
			&= \frac{1}{\alpha}\int_{1-\alpha+\epsilon}^{1+\epsilon} sup\{z|F_Y(z)\leq \tau\} d\tau\\
			&=\frac{1}{\alpha}[\int_{1-\alpha+\epsilon}^{1} sup\{z|F_Y(z)\leq \tau\} d\tau + \int_{1}^{1+\epsilon} sup\{z|F_Y(z)\leq 1\}d\tau]\triangleq A
		\end{aligned}
	\end{equation}
	If $\epsilon < \alpha$ then
	$$
	\begin{aligned}
		A&=\frac{1}{\alpha}\int_{1-\alpha+\epsilon}^{1} sup\{z|F_Y(z)\leq \tau\} d\tau + \frac{\epsilon}{\alpha} \times sup \quad Img(Y) \\
		&=\frac{\alpha-\epsilon}{\alpha}\frac{1}{\alpha-\epsilon}\int_{1-\alpha+\epsilon}^{1} sup\{z|F_Y(z)\leq \tau\} d\tau + \frac{\epsilon}{\alpha} \times sup \quad Img(Y)\\
		&=\frac{\alpha-\epsilon}{\alpha}CVaR_{\alpha-\epsilon}(Y) + \frac{\epsilon}{\alpha} \times sup \quad Img(Y)
	\end{aligned}
	$$
	If $\epsilon \geq \alpha$ then
	$$
	A=\frac{\alpha}{\alpha}sup \quad Img(Y)=sup \quad Img(Y)
	$$
	\begin{equation}
		\begin{aligned}
			CVaR_\alpha(X)&=\frac{1}{\alpha}\int_{1-\alpha}^1 sup\{z|F_X(z)\leq \tau\}d\tau\\
			&\geq \frac{1}{\alpha}\int_{1-\alpha}^1 sup\{z|F_Y(z)\leq \tau - \epsilon\}d\tau
			= \frac{1}{\alpha}\int_{1-\alpha-\epsilon}^{1-\epsilon} sup\{z|F_Y(z)\leq \tau\}d\tau \\
			&= \frac{1}{\alpha}[\int_{1-(\alpha+\epsilon)}^{1} sup\{z|F_Y(z)\leq \tau\}d\tau - \int_{1-\epsilon}^{1} sup\{z|F_Y(z)\leq \tau\}d\tau]\triangleq B
		\end{aligned}
	\end{equation}
	If $\epsilon + \alpha < 1$ then
	$$
	\begin{aligned}
		B&=\frac{\alpha+\epsilon}{\alpha}CVaR_{\alpha+\epsilon}(Y)-\frac{\epsilon}{\alpha}\frac{1}{\epsilon}\int_{1-\epsilon}^{1} sup\{z|F_Y(z)\leq \tau\}d\tau
		&=\frac{\alpha+\epsilon}{\alpha}CVaR_{\alpha+\epsilon}(Y)-\frac{\epsilon}{\alpha}CVaR_{\epsilon}(Y)
	\end{aligned}
	$$
	If $\epsilon + \alpha \geq 1$ then
	$$
	\begin{aligned}
		B&=\frac{1}{\alpha}[\int_{1-(\alpha+\epsilon)}^0 sup\{z|F_Y(z)\leq \tau\}d\tau+ \int_0^1 sup\{z|F_Y(z)\leq \tau\}d\tau]-\frac{\epsilon}{\alpha}CVaR_{\epsilon}(Y) \\
		&=\frac{1}{\alpha}[(\alpha+\epsilon-1)\inf Img(Y)+ E[Y]-\epsilon CVaR_{\epsilon}(Y)]
	\end{aligned}
	$$
\end{proof}

\begin{theorem}\label{prf:tight_cvar_lower_bound}
	(Tighter CVaR Lower Bound) Let $\alpha \in (0,1)$, $X$ and $Y$ be random variables. Define the a random variable $Y^L$ such that $F_{Y^L}(y) \triangleq min(1, F_Y(y) + g(y))$ for $g:\mathbb{R}\rightarrow [0, \infty)$. Assume 
	$lim_{x \rightarrow -\infty}g(x)=0$, $g$ is continuous from the right and monotonic increasing.
	If $\forall x\in \mathbb{R}, F_X(x)\leq F_Y(x)+g(x)$, then $F_{Y^L}$ is a CDF and $CVaR_\alpha(Y^L)\leq CVaR_\alpha(X).$
\end{theorem}
\begin{proof}
	In order to prove that $F_{Y^L}$ is a CDF we need to prove that $F_{Y^L}$ is:
	\begin{enumerate}
		\item Monotonic increasing
		\item $F:\mathbb{R}:\rightarrow [0,1]$, $\lim_{x\rightarrow \infty}F_{Y^L}(x)=1$, $\lim_{x\rightarrow -\infty}F_{Y^L}(x)=0$
		\item Continuous from the right
	\end{enumerate}
	
	\textbf{Monotonic increasing:}
	Note that for every $f_i:\mathbb{R}\rightarrow \mathbb{R}, i=1,2$ that are monotonic increasing, $f_1(f_2(x)))$ is also monotonic increasing in x. Denote $f(x):=min(x,1)$ and $f_2(x):=F_Y(x)+g(x)$. $F_Y(x)$ is a CDF and therefore monotonic increasing, so $f_2(x)$ is monotonic increasing as a sum of monotonic increasing functions. $f_1$ is also monotonic increasing, and $F_{Y^L}(x)=f_1(f_2(x))$. Therefore $F_{Y^L}(x)$ is monotonic increasing. \\
	
	\textbf{Limits:} $$
	1\geq\lim_{x\rightarrow \infty} F_{Y^L}(x)=\lim_{x\rightarrow \infty} min(1,F_Y(x)+g(x))\geq \lim_{x\rightarrow \infty}min(1,F_Y(x))=\lim_{x\rightarrow \infty} F_Y(x)=1
	$$
	and therefore $\lim_{x\rightarrow \infty} F_{Y^L}(x)=1$.
	$$\begin{aligned}
	0\leq \lim_{x\rightarrow -\infty} F_{Y^L}(x)=\lim_{x\rightarrow -\infty} min(1,F_Y(x)+g(x))&\leq \lim_{x\rightarrow -\infty} F_Y(x)+g(x) \\
	&=\lim_{x\rightarrow -\infty} F_Y(x)+\lim_{x\rightarrow -\infty} g(x)=0	
	\end{aligned}$$
	and therefore $\lim_{x\rightarrow -\infty} F_{Y^L}(x)=0$. By definition $\forall x\in \mathbb{R},F_{Y^L}(x)\leq 1$, and $\forall x\in \mathbb{R},F_{Y^L}(x)\geq 0$ because both g and $F_{Y^L}$ are non negative functions. \\
	
	\textbf{Continuity from the right:}
	$F_Y$ is continuous from the right because it is a CDF, and therefore $F_{Y^L}$ is continuous from the right as a sum of continuous from the right functions. \\
	Thus, $F_{Y^L}$ is a CDF. \\
	
	\textbf{Bound proof:}
	If $Y^L\leq X$, then $CVaR_\alpha(Y^L)\leq CVaR_\alpha(X)$ because CVaR is a coherent risk measure. Note that 
	if $F_Y(x) + g(x)<1$ then
	$$F_X(x)\leq F_Y(x) + g(x)=F_{Y^L}(x),$$
	and if $F_Y(x) + g(x)\geq 1$, $1=F_{Y^L}(x)\geq F_X(x)$. Therefore $Y^L\leq X$.
\end{proof}

\begin{theorem}\label{prf:lower_cvar_bound_using_density_bound}
	Let $\alpha \in (0,1)$, $X$ and $Y$ random variables. Define $h:\mathbb{R}\rightarrow [0,\infty)$ to be a continuous function, $g(z):=\int_{-\infty}^z h(x)dx$ and $Y^L$ to be a random variable such that $F_{Y^L}(y):=\min(1, F_{Y^L}(y) + g(y))$. If $\lim_{z\rightarrow -\infty} g(z)=0$ and $\forall x\in \mathbb{R},f_x(x)\leq f_y(z) + h(x)$, then $F_{Y^L}$ is a CDF and $CVaR_\alpha(Y^L)\leq CVaR_\alpha(X)$.
\end{theorem}
\begin{proof}
	We will show the g satisfies the properties of Theorem \ref{thm:tight_cvar_lower_bound}, and therefore this theorem holds. We need to prove that \begin{enumerate}
		\item $\lim_{z\rightarrow -\infty} g(z)=0$
		\item g is continuous from the right.
		\item g is monotonic increasing.
		\item $F_X(y)\leq F_Y(y) + g(y)$
	\end{enumerate}
	It is given in the theorem's assumptions that $\lim_{z\rightarrow -\infty} g(z)=0$, so (1) holds. h is non negative and therefore $g$ is monotonic increasing, so (3) holds. $g$ is continuous if its derivative exists for all $z\in \mathbb{R}$. Let $z\in \mathbb{R}$ and $a<z$.
	$$
	\frac{d}{dz}g(z) = \frac{d}{dz}\int_{-\infty}^z h(x)dx = \frac{d}{dz}[\int_{-\infty}^a h(x)dx + \int_{a}^z h(x)dx]=\frac{d}{dz}\int_{a}^z h(x)dx = h(z)
	$$ where the third equality holds because $\int_{-\infty}^a h(x)dx = g(a)$ is a constant that does not depend on z. The last equality holds from the fundamental theorem of calculus because $h$ is continuous. 	Finally, (4) holds because
	$$F_X(y) \triangleq \int_{-\infty}^y f_x(x)dx\leq \int_{-\infty}^y f_y(x) + h(x)dx \triangleq F_Y(y) + g(y).$$
\end{proof}

\subsection{Proofs of section \ref{sec:bounds_for_simplified_belief_model}}\label{sec:bounds_for_simplified_belief_model_proofs}

\begin{theorem}\label{prf:simplification_bound_over_belief_distribution}
	Denote a measure for the difference between two probability measures $P$ and $P_s$. Then the following holds
	$$
	\begin{aligned}
		&|P(R_{k:T} \leq l|b_k,\pi)-P_s(R_{k:T} \leq l|b_k,\pi)| \leq \sum_{i=k+1}^{T-1}  
		\mathbb{E}_{b_{{k+1}:i} \sim P_s}
		[1_{R_{k+1:i}\leq f(l,i)}\Delta^s(b_i, a_i)|b_k,\pi].
	\end{aligned}
	$$
\end{theorem}
\begin{proof}
	Define
	$$
	\Delta^s(b_{T-1}, a_{T-1}):=
	\int_{b_T \in B}|P(b_T|b_{T-1},a_{T-1}) - P_s(b_T|b_{T-1},a_{T-1})|db_T
	$$ \\
	to be a measure of difference between and simplified and theoretical belief transition probabilities. Denote $$
	g(t)\triangleq \int_{b_{k+1:t}\in B^{t-k}} 1_{R_{k:t}\leq l-c(b_k,a_k) + (T-t)R_{max}} |\prod_{i=k+1}^t P(b_i|b_{i-1}, a_i) - \prod_{i=k+1}^t P_s(b_i|b_{i-1}, a_i)|db_{k+1:t}.
	$$ 
	By conditioning over the beliefs path from time k+1 to time T we get
	\begin{equation}
		\begin{aligned}
			P(\sum_{t=k}^T c(b_t,a_t)\leq l|b_{k},\pi) & = \int_{b_{k+1:T}\in B^{T-k}} P(R_{k:T}\leq l|b_{k:T},\pi)\prod_{i=k+1}^T P(b_i|b_{i-1}, \pi, a_i)db_{k+1:T}\\
			& = \int_{b_{k+1:T}\in B^{T-k}} 1_{R_{k:T}\leq l}\prod_{i=k+1}^T P(b_i|b_{i-1}, \pi, a_i)db_{k+1:T}.
		\end{aligned}
	\end{equation}
	
	Note that $R_{k:T}$ is a constant given $\pi$ and $b_{k:T}$. Hence,
	\begin{equation}
		\begin{aligned}
			&P(R_{k:T}\leq l|b_k,\pi)-P_s(R_{k:T}\leq l|b_k,\pi)\\
			&=\int_{b_{k+1:T}\in B^{T-k}} 1_{R_{k:T}\leq l} [\prod_{i=k+1}^T P(b_i|b_{i-1}, \pi, a_i) - \prod_{i=k+1}^T P_s(b_i|b_{i-1}, \pi, a_i)]db_{k+1:T}
		\end{aligned}
	\end{equation}
	By applying the triangle inequality we get $$\begin{aligned}
		&|P(R_{k:T}\leq l|b_k,\pi)-P_s(R_{k:T}\leq l|b_k,\pi)|\\
		&\leq \int_{b_{k+1:T}\in B^{T-k}} 1_{R_{k:T}\leq l} |\prod_{i=k+1}^T P(b_i|b_{i-1}, \pi, a_i) - \prod_{i=k+1}^T P_s(b_i|b_{i-1}, \pi, a_i)|db_{k+1:T}=g(T)
	\end{aligned}$$
	From \ref{thm:recursive_relation_between_belief_mupltiplication} we know that 
	$$
	\begin{aligned}
		&\prod_{i=k+1}^T P(b_i|b_{i-1}, \pi, a_i) - \prod_{i=k+1}^T P_s(b_i|b_{i-1}, \pi, a_i)\\
		&= [P(b_T|b_{T-1},a_{T-1}) - P_s(b_T|b_{T-1},a_{T-1})] \prod_{i=k+1}^{T-1}P_s(b_i|b_{i-1},a_{i-1}) \\
		&\quad + P(b_T|b_{T-1},a_{T-1})[\prod_{i=k+1}^{T-1} P(b_i|b_{i-1}, \pi, a_i) - \prod_{i=k+1}^{T-1} P_s(b_i|b_{i-1}, \pi, a_i)].
	\end{aligned}
	$$
	By plugging this into g we get
	$$\begin{aligned}
		&g(T)=\int_{b_{k+1:T}\in B^{T-k}} 1_{R_{k:T}\leq l-c(b_k,a_k)} |\prod_{i=k+1}^T P(b_i|b_{i-1}, a_i) - \prod_{i=k+1}^T P_s(b_i|b_{i-1}, a_i)|db_{k+1:T} \\
		&\overset{1}{=}\int_{b_{k+1:T}\in B^{T-k}} 1_{R_{k:T}\leq l-c(b_k,a_k)} |[P(b_T|b_{T-1},a_{T-1}) - P_s(b_T|b_{T-1},a_{T-1})] \\
		&\times \prod_{i=k+1}^{T-1}P_s(b_i|b_{i-1},a_{i-1}) \\
		&\quad + P(b_T|b_{T-1},a_{T-1})[\prod_{i=k+1}^{T-1} P(b_i|b_{i-1}, \pi, a_i) - \prod_{i=k+1}^{T-1} P_s(b_i|b_{i-1}, \pi, a_i)]|db_{k+1:t} \\
		&\overset{2}{\leq} \underbrace{\int_{b_{k+1:T}\in B^{T-k}} \!\!\!\!\!\!\!\!\!\!\!\!\!\!\!\! 1_{R_{k:T}\leq l-c(b_k,a_k)} \prod_{i=k+1}^{T-1}P_s(b_i|b_{i-1},a_{i-1})|P(b_T|b_{T-1},a_{T-1}) - P_s(b_T|b_{T-1},a_{T-1})|db_{k+1:T}}_{A_1} \\
		&+\underbrace{\int_{b_{k+1:T}\in B^{T-k}} \!\!\!\!\!\!\!\!\!\!\!\!\!\!\!\! 1_{R_{k:T}\leq l-c(b_k,a_k)} P(b_T|b_{T-1},a_{T-1})|\prod_{i=k+1}^{T-1} P(b_i|b_{i-1}, a_i) - \prod_{i=k+1}^{T-1} P_s(b_i|b_{i-1}, \pi, a_i)|db_{k+1:T}}_{A_2}
	\end{aligned}$$
	$^1$ Theorem \ref{thm:recursive_relation_between_belief_mupltiplication} \\
	$^2$ Triangle inequality \\ \\
	The term $A_1$ is an expectation over the TV distance, with respect to the simplified distribution
	$$A_1=\mathbb{E}_{b_{k:T-1}\sim P_s}[1_{R_{k:T-1}\leq l-c(b_k,a_k)+R_{max}}\Delta^s(b_{T-1}, a_{T-1})].$$
	The term $A_2$ can be expressed using $g(T-1)$, and yields a recursive relation. 
	$$\begin{aligned}
		&A_2\overset{3}{\leq}\int_{b_{k+1:T}\in B^{T-k}} 1_{R_{k:T-1}\leq l-c(b_k,a_k)+R_{max}} P(b_T|b_{T-1},a_{T-1}) \\
		&\times |\prod_{i=k+1}^{T-1} P(b_i|b_{i-1}, a_i) - \prod_{i=k+1}^{T-1} P_s(b_i|b_{i-1}, \pi, a_i)|db_{k+1:T} \\
		&=\int_{b_{k+1:T}\in B^{T-k}} 1_{R_{k:T-1}\leq l-c(b_k,a_k)+R_{max}}|\prod_{i=k+1}^{T-1} P(b_i|b_{i-1}, a_i) - \prod_{i=k+1}^{T-1} P_s(b_i|b_{i-1}, \pi, a_i)| \\
		&\underbrace{\int_{b_T}P(b_T|b_{T-1},a_{T-1})db_T}_{=1}db_{k+1:T-1} \\
		&=\int_{b_{k+1:T}\in B^{T-k}} 1_{R_{k:T-1}\leq l-c(b_k,a_k)+R_{max}} \\
		&\times |\prod_{i=k+1}^{T-1} P(b_i|b_{i-1}, a_i) - \prod_{i=k+1}^{T-1} P_s(b_i|b_{i-1}, \pi, a_i)| db_{k+1:T-1}=g(T-1)
	\end{aligned}$$
	$^3$ $1_{R_{k+1:T}\leq l-c(b_k,a_k)}=1_{R_{k+1:T-1}\leq l-c(b_k,a_k)-c(b_T,a_T)}\leq 1_{R_{k+1:T-1}\leq l-c(b_k,a_k)+R_{max}}$ \\\\
	Hence, $$g(T)\leq g(T-1) + \mathbb{E}_{b_{k:T-1}\sim P_s}[1_{R_{k:T-1}\leq l-c(b_k,a_k)+R_{max}}\Delta^s(b_{T-1}, a_{T-1})].$$
	By following this recursive bound we get,
	$$\begin{aligned}
		g(T)&\leq g(T-1) + \mathbb{E}_{b_{k:T-1}\sim P_s}[1_{R_{k:T-1}\leq l-c(b_k,a_k)+R_{max}}\Delta^s(b_{T-1}, a_{T-1})] \\
		&\leq \sum_{i=k+1}^{T-1} \mathbb{E}_{b_{k:i}\sim P_s}[1_{R_{k:T-1}\leq l-c(b_k,a_k)+(T-i)R_{max}}\Delta^s(b_{T-1}, a_{T-1})]
	\end{aligned}$$
	when the last inequality holds when applying the recursive relation that is exhibited in the first inequality.
\end{proof}

\begin{theorem}\label{thm:recursive_relation_between_belief_mupltiplication}
	$$
	\begin{aligned}
		&\prod_{i=k+1}^T P(b_i|b_{i-1}, \pi, a_i) - \prod_{i=k+1}^T P_s(b_i|b_{i-1}, \pi, a_i)\\
		&= [P(b_T|b_{T-1},a_{T-1}) - P_s(b_T|b_{T-1},a_{T-1})] \prod_{i=k+1}^{T-1}P_s(b_i|b_{i-1},a_{i-1}) \\
		&\quad + P(b_T|b_{T-1},a_{T-1})[\prod_{i=k+1}^{T-1} P(b_i|b_{i-1}, \pi, a_i) - \prod_{i=k+1}^{T-1} P_s(b_i|b_{i-1}, \pi, a_i)].
	\end{aligned}
	$$
\end{theorem}

\begin{proof}
	The term in absolute value above can be expressed in terms of beliefs probability subtraction. For all $x_i,y_i\in \mathbb{R}$,
	
	$$
	\prod_{i=1}^T x_i-\prod_{i=1}^{T}y_i=\prod_{i=1}^T x_i-\prod_{i=1}^{T}y_i + x_T\prod_{i=1}^{T}y_i-x_T\prod_{i=1}^{T}y_i=(x_T-y_T) \prod_{i=1}^{T-1}y_i+x_T(\prod_{i=1}^{T-1} x_i-\prod_{i=1}^{T-1}y_i).
	$$\\
	
	By denoting $x_i=P(b_i|b_{i-1},a_{i-1})$ and $y_i=P_s(b_i|b_{i-1},a_{i-1})$ we get
	
	$$
	\begin{aligned}
		&\prod_{i=k+1}^T P(b_i|b_{i-1}, \pi, a_i) - \prod_{i=k+1}^T P_s(b_i|b_{i-1}, \pi, a_i)\\
		&= [P(b_T|b_{T-1},a_{T-1}) - P_s(b_T|b_{T-1},a_{T-1})] \prod_{i=k+1}^{T-1}P_s(b_i|b_{i-1},a_{i-1}) \\
		&\quad + P(b_T|b_{T-1},a_{T-1})[\prod_{i=k+1}^{T-1} P(b_i|b_{i-1}, \pi, a_i) - \prod_{i=k+1}^{T-1} P_s(b_i|b_{i-1}, \pi, a_i)].
	\end{aligned}
	$$
\end{proof}

\begin{theorem}\label{prf:tight_lower_bound_for_v_and_q}
	(Tighter Lower Bound for V and Q) Let $\alpha \in (0,1)$, $k,T\in \mathbb{N}$ such that $k<T$, belief $b_k\in B$, action $a_k\in A$ and policy $\pi:X\rightarrow A$. Denote 
	$$\begin{aligned}
		g(l)\triangleq \sum_{i=k+1}^{T-1} \underset{b_{{k+1}:i}}{\mathbb{E}}
		[1_{R_{k+1:i}\leq l -c(b_k,a_k)+ (T-i)R_{max}} 
		\Delta^s(b_i, a_i)|b_k,\pi].
	\end{aligned}$$
	Let $P$ and $P_s$ be two probability measures, and define the 
	random variable $Y^L$ such that $$F_{Y^L}(y):=min(1, P_s(R_{k:T}\leq y|b_k,a_k,\pi) + g(y)).$$
	Then,
	\begin{enumerate}
		\item $F_{Y^L}$ is a CDF.
		\item $V^\pi_P(b_k, \alpha)\geq CVaR_\alpha^{P_s}[Y^L|b_k,\pi]$ and $Q^\pi_P(b_k, a_k,\alpha)\geq CVaR_\alpha^{P_s}[Y^L|b_k,a_k,\pi]$
	\end{enumerate}
\end{theorem}
\begin{proof}
	Theorem \ref{thm:tight_cvar_lower_bound} shows how to construct a lower bound for a random variable X, using a random variable Y and a function g. We will show the $g(l)$ that is defined in this theorem, satisfies the properties of the $g$ function that is defined in Theorem \ref{thm:tight_cvar_lower_bound}. Hence, using $g$ and the simplified distribution $P_s$, we can bound the CVaR that is computed using P. We will prove that g is 
	\begin{enumerate}
		\item Monotonic increasing
		\item Continuous from the right
		\item $\lim_{x\rightarrow -\infty}g(x)=0$
		\item $\forall x\in \mathbb{R},g(x)\geq 0$
	\end{enumerate}
	
	\textbf{Monotonic increasing}: Let $l\geq 0,h > 0$,
	$$\begin{aligned}
		g(l+h) - g(h) &= \sum_{i=k+1}^{T-1} \mathbb{E}_{b_{{k+1}:i}}
		[1_{R_{k+1:i}\leq f(l+h,i)}\Delta^s(b_i, a_i)|b_k,\pi] \\
		&- \sum_{i=k+1}^{T-1} \mathbb{E}_{b_{{k+1}:i}}
		[1_{R_{k+1:i}\leq f(l,i)}\Delta^s(b_i, a_i)|b_k,\pi] \\
		&=\sum_{i=k+1}^{T-1} \mathbb{E}_{b_{{k+1}:i}} [1_{f(l,i) \leq R_{k+1:i}\leq f(l+h,i)}\Delta^s(b_i, a_i)|b_k,\pi]\geq 0
	\end{aligned}$$
	when the last inequality holds because $\Delta^s(b_i, a_i)\geq 0$ and the indicator function is non negative. Hence, $g$ is monotonic increasing.\\
	
	\textbf{Limits}:
	For all $y\in \mathbb{R}$ such that $$y-c(b_k,a_k)+ (T-i)R_{max}<-R_{max}(T-k+1),$$ it holds that
	$$\begin{aligned}
		g(y)&\leq g(-R_{max}(T-k+1))=\sum_{i=k+1}^{T-1} E_{b_{k+1:i}}[1_{R_{k+1:i}\leq -R_{max}(T-k+1)}\Delta^s(b_i,a_i)|b_k,\pi] \\
		&=\sum_{i=k+1}^{T-1} E_{b_{k+1:i}}[0\times\Delta^s(b_i,a_i)|b_k,\pi]=0
	\end{aligned}$$
	when the first inequality holds because g is monotonic increasing and the second equality holds because $|R_{k+1:i}|<R_{max}(T-k+1)$. Hence, $\lim_{x\rightarrow -\infty} g(x)\leq 0$, and because $\forall l\geq 0,g(l)\geq 0$, we get that $\lim_{x\rightarrow -\infty} g(x)=0$\\
	
	\textbf{Non negativity of g:} $\Delta(b_i,a_i)\geq 0$ and the indicator function is non negative. \\
	
	\textbf{Continuity from the right}:
	We need to show that $\lim_{l\rightarrow l_0^+}g(l)=g(l_0)$. Let $l_n\in \mathbb{R}$ such that $n\in \mathbb{N}$ and $l_n\rightarrow l_0^+$ (that is, $l_n$ converges to $l_0$ from the right). Denote 
	$$f_n(b_{k+1},\dots,b_i)\triangleq 1_{R_{k+1:i}\leq l_n -c(b_k,a_k)+ (T-i)R_{max}}\Delta^s(b_i, a_i),$$
	$$h_n(b_{k+1},\dots,b_i)\triangleq 2,$$
	for $n\geq 1$. Note that $\Delta^s(b_i, a_i)$ is constant with respect to $l_n$ and $1_{R_{k+1:i}\leq l_n -c(b_k,a_k)+ (T-i)R_{max}}$ is continuous from the right with respect to $l_n$. Hence, $\lim_{n\rightarrow \infty}f_n(b_{k+1},\dots,b_i)=f_0(b_{k+1},\dots,b_i)$. We also get that $f_n(b_{k+1},\dots,b_i)\leq h_n(b_{k+1},\dots,b_i)$ because 
	$$\begin{aligned}
		\Delta(b_{i-1},a_{i-1})&\triangleq\int_{b_{i}}|P(b_i|b_{i-1},a_{i-1}) - P_s(b_i|b_{i-1},a_{i-1})|db_i \\
		&\leq \int_{b_{i}}P(b_i|b_{i-1},a_{i-1}) + P_s(b_i|b_{i-1},a_{i-1})db_i=2.
	\end{aligned}$$
	It holds that $\int_{b_{k+1:i}} h_n(b_{k+1},\dots,b_i)dP(b_{k+1:i}|b_k,\pi)=2$ and therefore from the dominant convergence theorem we get 
	$$\lim_{n\rightarrow \infty} \int_{b_{k+1:i}} f_n(b_{k+1},\dots,b_i)dP(b_{k+1:i}|b_k,\pi)=\int_{b_{k+1:i}} f_0(b_{k+1},\dots,b_i)dP(b_{k+1:i}|b_k,\pi).$$
	Until now we proved that $\mathbb{E}_{b_{{k+1}:i}}
	[1_{R_{k+1:i}\leq l -c(b_k,a_k)+ (T-i)R_{max}}\Delta^s(b_i, a_i)|b_k,\pi]$ is continuous from the right. Therefore, $g$ is continuous from the right with respect to l as a sum of functions that are continuous from the right.
	From \ref{thm:simplification_bound_over_belief_distribution} we get 
	$$\forall l\in \mathbb{R},P(R_{k:T}\leq l|b_k,a_k,\pi)\leq P_s(R_{k:T}\leq l|b_k,a_k,\pi)+g(l),$$ 
	and from \ref{thm:tight_cvar_lower_bound} we get what we want to prove.
\end{proof}

\begin{theorem}\label{prf:uniform_lower_and_upper_bounds_for_v_and_q}
	Denote $\epsilon\triangleq\sum_{i=k+1}^{T-1} \mathbb{E}_{b_{{k+1}:i}}[\Delta^s(b_i, a_i)|b_k,\pi]$.
	\begin{enumerate}
		\item \begin{enumerate}
			\item If $\epsilon<\alpha$, then 
			$U_s\triangleq\frac{\alpha-\epsilon}{\alpha}Q^\pi_{P_s}(b_k,a_k,\alpha-\epsilon) + \frac{\epsilon}{\alpha}R_{max}(T-k+1)$
			\item If $\epsilon \geq \alpha$ then $U_s\triangleq(T-k+1)R_{max}$
		\end{enumerate}
		\item \begin{enumerate}
			\item If $\epsilon + \alpha < 1$ then 
			$
			L_s\triangleq\frac{\alpha+\epsilon}{\alpha}Q^\pi_{P_s}(b_k,a_k,\alpha+\epsilon) - \frac{\epsilon}{\alpha} Q^\pi_{P_s}(b_k,a_k, \epsilon)
			$
			\item If $\epsilon + \alpha \geq 1$ then 
			$
			L_s\triangleq\frac{1}{\alpha}[-(\alpha + \epsilon - 1)(T-k+1)R_{max}+Q^\pi_{P_s}(b_k,a_k)-\epsilon Q^\pi_{P_s}(b_k,a_k,\epsilon)]
			$
		\end{enumerate}
	\end{enumerate}
	Then $L_s \leq V_P^\pi(b_k,\alpha)\leq U_s$ and $L_s \leq Q_P^\pi(b_k, a_k,\alpha)\leq U_s$.
\end{theorem}
\begin{proof}
	For all $l\in \mathbb{R}$
	\begin{equation}
		\begin{aligned}
			|P(R_{k:T} \leq l|b_k,\pi)-P_s(R_{k:T} \leq l|b_k,\pi)| &\leq \sum_{i=k+1}^{T-1} \mathbb{E}_{b_{{k+1}:i}}
			[1_{R_{k:i}\leq l -c(b_k,a_k)+ (T-i)R_{max}}\Delta^s(b_i, a_i)|b_k,\pi]\\
			&\leq \sum_{i=k+1}^{T-1} \mathbb{E}_{b_{{k+1}:i}}
			[\Delta^s(b_i, a_i)|b_k,\pi]=\epsilon,
		\end{aligned}
	\end{equation}
	
	when the first inequality holds from \ref{thm:simplification_bound_over_belief_distribution}. From \ref{thm:cvar_bound} we get
	
	\begin{enumerate}
		\item If $\epsilon + \alpha < 1$,
		$$CVaR_{\alpha}^P(R_{k:T}|b_k,a_k,\pi) \geq \frac{\alpha + \epsilon}{\alpha}CVaR_{\alpha+\epsilon}^{P_s}(R_{k:T}|b_k,a_k,\pi)-\frac{\epsilon}{\alpha}CVaR_{\epsilon}^{P_s}(R_{k:T}|b_k,a_k,\pi)$$ 
		and if $\epsilon + \alpha \geq 1$ then 
		$$
		\begin{aligned}
			CVaR_\alpha(R_{k:T}|b_k,a_k,\pi) &\geq \frac{1}{\alpha}[(\alpha+\epsilon-1)\inf \text{Img}(R_{k:T}|b_k,a_k,\pi) + E[R_{k:T}|b_k,a_k,\pi] \\
			&- \epsilon CVaR_\epsilon (R_{k:T}|b_k,a_k,\pi)] \\
			&\geq \frac{1}{\alpha}[(\alpha+\epsilon-1)(-R_{max}(T-k+1)) + E[R_{k:T}|b_k,a_k,\pi] \\
			&- \epsilon CVaR_\epsilon (R_{k:T}|b_k,a_k,\pi)]
		\end{aligned}
		$$
		
		\item If $\epsilon<\alpha$ then $$
		\begin{aligned}
			CVaR_\alpha (R_{k:T}|b_k,a_k,\pi)&\leq \frac{\alpha - \epsilon}{\alpha}CVaR_{\alpha-\epsilon}(R_{k:T}|b_k,a_k,\pi)+\frac{\epsilon}{\alpha}\sup Img(R_{k:T}|b_k,a_k,\pi) \\
			&\leq \frac{\alpha - \epsilon}{\alpha}CVaR_{\alpha-\epsilon}(R_{k:T}|b_k,a_k,\pi)+\frac{\epsilon}{\alpha}R_{max}(T-k+1)
		\end{aligned}
		$$ and if $\epsilon\geq \alpha$ then $$
		CVaR_\alpha (R_{k:T}|b_k,a_k,\pi) \leq \sup Img(R_{k:T}|b_k,a_k,\pi)\leq (T-k+1)R_{max}
		$$
	\end{enumerate}
	
	This proves that $L_s \leq Q_P^\pi(b_k,\alpha)\leq U_s$, and when $a_k=\pi(b_k)$ we get $L_s \leq V_P^\pi(b_k,\alpha)\leq U_s$.
\end{proof}

\subsection{Proofs of section \ref{sec:cdf_bound_estimation}}\label{sec:cdf_bound_estimation_proofs}

\begin{theorem}\label{prf:m_i_sum_bound_guarantees}
	Let $v>0,\delta\in (0,1)$. If $\hat{\Delta}$ is unbiased, then for $N_\Delta\geq -8B^2 \frac{ln(\frac{\delta}{4(T-1-k)})}{(v/(T-1-k))^2}$ it holds that
	$$P(|\hat{\epsilon}-\epsilon|>2v)\leq \delta,$$ for $B=\max_{i\in \{k+1,\dots,T\}} B_i$ when $B_i=\sup_{b_i}P_s(b_i|\bar{b}_k)/Q_0(b_i)$.
\end{theorem}
\begin{proof}
	$$\begin{aligned}
		P(|\sum_{i=k+1}^T \hat{m}_i-\sum_{j=k+1}^T m_j|>2(T-k)v)&\leq P(\sum_{i=k+1}^T |\hat{m}_i-m_i|>2(T-k)v) 
		\\
		&\leq \sum_{i=k+1}^T P(|\hat{m}_i-m_i|>2v)
	\end{aligned}$$
	when the first inequality is from the triangle inequality and the second inequality is from \ref{thm:prob_x_plus_y_bound}.
	From \ref{thm:m_i_guarantees}, $P(|\hat{m}_i-m_i|>2v)\leq \delta$ when $N_\Delta\geq -8B_i^2 ln(\delta/4)/v^2$. Hence, $P(|\hat{m}_i-m_i|>2v)\leq \delta$ also holds when $N_\Delta\geq -8B^2 ln(\delta/4)/v^2\geq -8B_i^2 ln(\delta/4)/v^2$. Therefore, 
	$$P(|\sum_{i=k+1}^T \hat{m}_i-\sum_{j=k+1}^T m_j|>2(T-k)v)\leq (T-k)\delta$$ when $N_\Delta\geq -8B^2 ln(\delta/4)/v^2$. Equivalently, for $N_\Delta\geq -8B^2 \frac{ln(\frac{\delta}{4(T-k)})}{(v/(T-k))^2}$ we get $$P(|\sum_{i=k+1}^T \hat{m}_i-\sum_{j=k+1}^T m_j|>2v)\leq \delta.$$
	Note that 
	$$P(|\hat{\epsilon}-\epsilon|>2v) = P(|\sum_{i=k+1}^{T-1} \hat{m}_i-\sum_{j=k+1}^{T-1} m_j|>2(T-1-k)v),$$
	and therefore by substituting T with T-1 in the results above we get what we want to prove.
\end{proof}

\begin{theorem}\label{thm:m_i_guarantees}
	Let $v>0,\delta\in (0,1)$. If $\hat{\Delta}$ is unbiased, then for $N_\Delta\geq -8B_i^2 ln(\delta/4)/v^2$ it holds that
	$$P(|\hat{m}_i-m_i|>2v)\leq \delta,$$ for $B_i=\sup_{b_i}P_s(b_i|\bar{b}_k)/Q_0(b_i)$.
\end{theorem}
\begin{proof}
	$$\begin{aligned}
		P(|\hat{m}_i-m_i|\geq 2v)&=P(|\hat{m}_i - \frac{1}{N_\Delta}\sum_{n=1}^{N_\Delta}\frac{P_s(b_i^{\Delta,n}|\bar{b}_{k},\pi)}{Q_0(b_i^{\Delta,n})}\Delta^s(b_i^{\Delta,n},\pi(b_i^{\Delta,n})) \\
		&+ \frac{1}{N_\Delta}\sum_{n=1}^{N_\Delta}\frac{P_s(b_i^{\Delta,n}|\bar{b}_{k},\pi)}{Q_0(b_i^{\Delta,n})}\Delta^s(b_i^{\Delta,n},\pi(b_i^{\Delta,n})) - m_i|\geq 2v) \\
		&\overset{\text{Theorem } \ref{thm:prob_x_plus_y_bound}}{\leq} \underbrace{P(|m_i-\frac{1}{N_\Delta}\sum_{n=1}^{N_\Delta}\frac{P_s(b_i^{\Delta,n}|\bar{b}_{k},\pi)}{Q_0(b_i^{\Delta,n})}\Delta^s(b_i^{\Delta,n},\pi(b_i^{\Delta,n}))|\geq v)}_{A} \\ 
		&+ \underbrace{P(\frac{1}{N_\Delta}\sum_{n=1}^{N_\Delta}\frac{P_s(b_i^{\Delta,n}|\bar{b}_{k},\pi)}{Q_0(b_i^{\Delta,n})}(\Delta^s(b_i^{\Delta,n},\pi(b_i^{\Delta,n})-\hat{\Delta}^s(b_i^{\Delta,n},\pi(b_i^{\Delta,n})))|\geq v)}_{B}
	\end{aligned}$$
	
	From Hoeffding's inequality,
	$$\begin{aligned}
		A&=P(|N_\Delta m_i-\sum_{n=1}^{N_\Delta}\frac{P_s(b_i^{\Delta,n}|\bar{b}_{k},\pi)}{Q_0(b_i^{\Delta,n})}\Delta^s(b_i^{\Delta,n},\pi(b_i^{\Delta,n}))|\geq vN_\Delta)\leq 2exp(-\frac{2v^2 N_\Delta^2}{N_\Delta B_i^2}) \\
		&=2exp(-\frac{2v^2 N_\Delta}{B_i^2}) \leq 2exp\{-\frac{2v^2(-ln(\delta/4)8B_i^2/v^2)}{B_i^2}\}= 2exp\{16ln(\delta/4)\}\leq \delta/2
	\end{aligned}
	$$
	$$
	\begin{aligned}
		B&=P(\frac{1}{N_\Delta}\sum_{n=1}^{N_\Delta}\frac{P_s(b_i^{\Delta,n}|\bar{b}_{k},\pi)}{Q_0(b_i^{\Delta,n})}(\Delta^s(b_i^{\Delta,n},\pi(b_i^{\Delta,n})-\hat{\Delta}^s(b_i^{\Delta,n},\pi(b_i^{\Delta,n})))|\geq v) \\
		&=P(\sum_{n=1}^{N_\Delta}\frac{P_s(b_i^{\Delta,n}|\bar{b}_{k},\pi)}{Q_0(b_i^{\Delta,n})}(\Delta^s(b_i^{\Delta,n},\pi(b_i^{\Delta,n})-\hat{\Delta}^s(b_i^{\Delta,n},\pi(b_i^{\Delta,n})))|\geq vN_\Delta) \\
		&\leq 2exp(-\frac{2v^2N_\Delta^2}{16 N_\Delta B_i^2})=2exp(-\frac{v^2N_\Delta}{8B_i^2})\leq 2exp(-\frac{v^2(-log(\delta/4)8B_i^2/v^2)}{8B_i^2}) \\
		&=2exp(log(\delta/4))=\delta/2
	\end{aligned}
	$$ when the use the Hoeffding's inequality in the first inequality is possible because $\hat{\Delta}$ is an unbiased estimator. Hence, $P(|m_i-\hat{m}_i|\geq 2v)\leq A+B\leq \delta$.
\end{proof}

\begin{theorem}\label{prf:g_estimator_convergence_guarantees}
	Let $v>0,\delta \in (0,1),l\in \mathbb{R}$ and denote $B_i \triangleq \sup_{b_i} \frac{P(b_i|\bar{b}_k, \pi)}{Q_0{b_i}}, B \triangleq \max_{i\in \{k+1,\dots,T\}} B_i$. If $N_\Delta \geq -ln(\frac{\delta/(T-k)}{2})\frac{2B^2}{v^2 / (T-k)^2}$ and $\hat{\Delta}$ is unbiased then $$P(|\sum_{i=k+1}^T g_i(l) - \sum_{i=k+1}^T \hat{g}_i(l)|>v)\leq \delta.$$
\end{theorem}
\begin{proof} From Theorem \ref{thm:prob_x_plus_y_bound}
	$$P(|\sum_{i=k+1}^T g_i(l) - \sum_{i=k+1}^T \hat{g}_i(l)|>v)\leq \sum_{i=k+1}^T P(|g_i(l) - 
	\hat{g}_i(l)|>\frac{v}{T-k}).$$
	Note that $N_\Delta = -ln(\frac{\delta/(T-k)}{2})\frac{2B^2}{v^2 / (T-k)^2} \geq -ln(\frac{\delta/(T-k)}{2})\frac{2B_i^2}{v^2 / (T-k)^2}$, so from Theorem \ref{thm:g_i_convergence_guarantees} we get $$P(|g_i(l) - 
	\hat{g}_i(l)|>\frac{v}{T-k})\leq \frac{\delta}{T-k}.$$
	Hence, $$P(|\sum_{i=k+1}^T g_i(l) - \sum_{i=k+1}^T \hat{g}_i(l)|>v)\leq \frac{\delta}{(T-k)}(T-k)=\delta.$$
\end{proof}

\begin{theorem}\label{thm:g_i_convergence_guarantees}
	Let $v>0,\delta \in (0,1)$. If $N_\Delta \geq -ln(\frac{\delta}{2})\frac{2B_i^2}{v^2}$ then
	$P(|g_i(l)-\hat{g}_i(l)|>v)\leq \delta$, for $B_i=\sup_{b_i} \frac{P(b_i|\bar{b}_k, \pi)}{Q_0{b_i}}$.
\end{theorem}
\begin{proof}
	We will use Hoeffding's inequality to prove the bound.
	$$\begin{aligned}
		&E[\frac{P(b_i^{\Delta,n}|\bar{b}_{k},\pi)}{Q_0(b_i^{\Delta,n})}\hat{\Delta}^s(b_i^{\Delta,n},\pi(b_i^{\Delta,n}))1_{\bar{R}^n_{k+1:i}\leq l -c(b_k,a_k)+ (T-i)R_{max}}] \\
		&=E[E_{b_i^{\Delta,n} \sim Q_0}[\frac{P(b_i^{\Delta,n}|\bar{b}_{k},\pi)}{Q_0(b_i^{\Delta,n})}\hat{\Delta}^s(b_i^{\Delta,n},\pi(b_i^{\Delta,n}))|\bar{b}_{k:T}]1_{\bar{R}_{k+1:i}\leq l -c(b_k,a_k)+ (T-i)R_{max}}] \\
		&=E[E[\hat{\Delta}^s(b_i^{\Delta,n},\pi(b_i^{\Delta,n}))|\bar{b}_k,\pi]1_{\bar{R}_{k+1:i}\leq l -c(b_k,a_k)+ (T-i)R_{max}}] \\
		&\overset{1}{=}E[\Delta^s(b_i^{\Delta,n},\pi(b_i^{\Delta,n}))1_{R_{k+1:i}\leq l -c(b_k,a_k)+ (T-i)R_{max}}|\bar{b}_k,\pi]
	\end{aligned}$$ 
	$^1$ $\hat{\Delta}$ is an unbiased estimator. \\
	Note that $$\begin{aligned}
		0 \leq \frac{P(b_i^{\Delta,n}|\bar{b}_{k},\pi)}{Q_0(b_i^{\Delta,n})}\hat{\Delta}^s(b_i^{\Delta,n},\pi(b_i^{\Delta,n}))1_{\bar{R}^n_{k+1:i}\leq l -c(b_k,a_k)+ (T-i)R_{max}}\leq 2B_i.
	\end{aligned}$$
	Hence, from Hoeffding's inequality,
	$$\begin{aligned}
		&P(|g_i(l)-\hat{g}_i(l)|>v|\bar{b}_k,\pi)\\
		&=P(|N_\Delta g_i(l) - \sum_{n=1}^{N_\Delta}\frac{P(b_i^{\Delta,n}|\bar{b}_{k},\pi)}{Q_0(b_i^{\Delta,n})}\hat{\Delta}^s(b_i^{\Delta,n},\pi(b_i^{\Delta,n}))1_{\bar{R}^n_{k+1:i}\leq f(l,i)}|>vN_\Delta) \\
		&\leq P(|N_\Delta g_i(l) - \sum_{n=1}^{N_\Delta}\frac{P(b_i^{\Delta,n}|\bar{b}_{k},\pi)}{Q_0(b_i^{\Delta,n})}\hat{\Delta}^s(b_i^{\Delta,n},\pi(b_i^{\Delta,n}))1_{\bar{R}^n_{k+1:i}\leq f(l,i)}|>-v \frac{2B_i^2ln(\frac{\delta}{2})}{v^2}) \\
		&\leq 2exp\{-2\frac{4B_i^4ln(\delta/2)^2}{4B_i^2 v^2 N_\Delta}\}
		=2exp\{-\frac{2B_i^2ln(\delta/2)^2}{v^2 N_\Delta}\}\leq 2exp\{-\frac{2B_i^2ln(\delta/2)^2}{-ln(\delta/2)\frac{2B_i^2}{v^2} v^2}\}\\
		&=2exp\{ln(\delta/2)\}=\delta
	\end{aligned}$$
	
\end{proof}

\begin{theorem}\label{prf:h_guarantees}
	Let $\alpha,\delta\in (0,1), v>0$. If $N_\Delta \geq -ln(\frac{(\delta/I)/(T-1-k)}{2})\frac{2B^2}{v^2 / (T-1-k)^2}$ then
	$$P(\sup_{x\in \mathbb{R}}\{g(x)-\hat{h}^+(x)\}>v|\bar{b}_k,a_k,\pi)\leq \delta$$
	$$P(\sup_{x\in \mathbb{R}}\{\hat{h}^-(x)-g(x)\}>v|\bar{b}_k,a_k,\pi)\leq \delta$$
\end{theorem}
\begin{proof}
	$$\begin{aligned}
		&P(\sup_{x\in \mathbb{R}}\{g(x)-\hat{h}^+(x)\}>v|\bar{b}_k,a_k,\pi)=P(\max_{1\leq i \leq I}\sup_{x\in (k_{i-1},k_i]}\{g(x)-\hat{h}^+(x)\}>v|\bar{b}_k,a_k,\pi) \\
		&\overset{1}{=}P(\max_{1\leq i \leq I}\sup_{x\in (k_{i-1},k_i]}\{g(x)-\hat{g}(k_i)\}>v|\bar{b}_k,a_k,\pi)\\
		&\overset{2}{\leq} P(\max_{1\leq i \leq I}\sup_{x\in (k_{i-1},k_i]}\{g(k_i)-\hat{g}(k_i)\}>v|\bar{b}_k,a_k,\pi) \\
		&=P(\max_{1\leq i \leq I}\{g(k_i)-\hat{g}(k_i)\}>v|\bar{b}_k,a_k,\pi)=P(\cup_{1\leq i \leq I}\{g(k_i)-\hat{g}(k_i)>v\}|\bar{b}_k,a_k,\pi) \\
		&\leq \sum_{i=1}^I P(g(k_i)-\hat{g}(k_i)>v|\bar{b}_k,a_k,\pi) \\
		&=\sum_{i=1}^I \underbrace{P(g(k_i)-\hat{g}(k_i)>v|\bar{b}_k,a_k,\pi, |g(k_i)-\hat{g}(k_i)|>v)}_{\leq 1}\underbrace{P(|g(k_i)-\hat{g}(k_i)|>v|\bar{b}_k,a_k,\pi)}_{\leq \delta/I} \\
		&+ P(g(k_i)-\hat{g}(k_i)>v|\bar{b}_k,a_k,\pi, |g(k_i)-\hat{g}(k_i)|\leq v)\underbrace{P(g(k_i)-\hat{g}(k_i)\leq v|\bar{b}_k,a_k,\pi)}_{\leq 1} \\
		&\overset{3}{\leq} \delta + \sum_{i=1}^I P(g(k_i)-\hat{g}(k_i)>v|\bar{b}_k,a_k,\pi, |g(k_i)-\hat{g}(k_i)|\leq v) \\
		&\leq \delta + \sum_{i=1}^I \underbrace{P(g(k_i)-(g(k_i)-v)>v|\bar{b}_k,a_k,\pi, |g(k_i)-\hat{g}(k_i)|\leq v)}_{=0} \\ 
		&=\delta
	\end{aligned}$$
	$^1$ Definition of $\hat{h}^+$. \\
	$^2$ g is monotomic increasing. \\
	$^3$ Theorem \ref{thm:g_estimator_convergence_guarantees} guarantees that for $N_\Delta \geq -ln(\frac{(\delta/I)/(T-1-k)}{2})\frac{2B^2}{v^2 / (T-1-k)^2}$, $\forall l\in \mathbb{R},P(|g(l)-\hat{g}(l)|>v)\leq \delta/I$.
	
	$$\begin{aligned}
		&P(\sup_{x\in \mathbb{R}}\{\hat{h}^-(x)-g(x)\}>v|\bar{b}_k,a_k,\pi)=P(\max_{1\leq i \leq I}\sup_{x\in (k_{i-1},k_i]}\{\hat{h}^-(x) - g(x)\}>v|\bar{b}_k,a_k,\pi) \\
		&\overset{4}{=}P(\max_{1\leq i \leq I}\sup_{x\in (k_{i-1},k_i]}\{\hat{g}(k_{i-1}) - g(x)\}>v|\bar{b}_k,a_k,\pi) \\
		&\overset{5}{\leq} P(\max_{1\leq i \leq I}\sup_{x\in (k_{i-1},k_i]}\{\hat{g}(k_{i-1}) - g(k_{i-1})\}>v|\bar{b}_k,a_k,\pi) \\
		&=P(\max_{1\leq i \leq I}\{\hat{g}(k_{i-1}) - g(k_{i-1})\}>v|\bar{b}_k,a_k,\pi)=P(\cup_{1\leq i \leq I}\{\hat{g}(k_{i-1}) - g(k_{i-1})>v\}|\bar{b}_k,a_k,\pi) \\
		&\leq \sum_{i=1}^I P(\hat{g}(k_{i-1}) - g(k_{i-1})>v|\bar{b}_k,a_k,\pi) \\
		&=\sum_{i=1}^I \underbrace{P(\hat{g}(k_{i-1}) - g(k_{i-1})>v|\bar{b}_k,a_k,\pi, |g(k_{i-1})-\hat{g}(k_{i-1})|> v)}_{\leq 1} \\
		&\times \underbrace{P(|g(k_{i-1})-\hat{g}(k_{i-1})|> v|\bar{b}_k,a_k,\pi)}_{\leq \delta/I} \\
		&+ P(\hat{g}(k_{i-1}) - g(k_{i-1})>v|\bar{b}_k,a_k,\pi, |g(k_{i-1})-\hat{g}(k_{i-1})|\leq v)
		\underbrace{P(|g(k_{i-1})-\hat{g}(k_{i-1})|\leq v|\bar{b}_k,a_k,\pi)}_{\leq 1} \\
		&\overset{3}{\leq} \delta + \sum_{i=1}^I P(\hat{g}(k_{i-1}) - g(k_{i-1})>v|\bar{b}_k,a_k,\pi, |g(k_{i-1})-\hat{g}(k_{i-1})|\leq v) \\
		&\leq \delta + \sum_{i=1}^I \underbrace{P(v+g(k_{i-1}) - g(k_{i-1})>v|\bar{b}_k,a_k,\pi, |g(k_{i-1})-\hat{g}(k_{i-1})|\leq v)}_{=0} \\
		&=\delta
	\end{aligned}$$
	$^4$ Definition of $\hat{h}^-$. \\
	$^5$ g is continuous from the right, and therefore $\lim_{x\rightarrow k_{i-1}^+}g(x)=g(k_{i-1})$. g is also monotonic increasing, so $\forall x\in (k_{i-1},k_i],g(x)\geq g(k_{i-1})$.
\end{proof}

\subsection{Proofs of section \ref{sec:performance_guarantees}}\label{sec:performance_guarantees_proofs}

\begin{theorem}\label{prf:unfirom_bound_convergence_guarantees_with_estimated_epsilon}
	Let $\delta \in (0,1),\alpha\in (0,1),v>0$. Denote \begin{enumerate}
		\item $L_1 \triangleq \frac{\alpha+\hat{\epsilon}-4v}{\alpha}\hat{Q}^\pi_{M_{P_s}}(\bar{b}_k,a_k,\alpha+\hat{\epsilon}) - \frac{\hat{\epsilon}}{\alpha}\hat{Q}^\pi_{M_{P_s}}(\bar{b}_k,a_k,\hat{\epsilon}-4v)$
		
		\item $L_2 \triangleq \frac{1}{\alpha}[\hat{Q}^\pi_{M_{P_s}}(\bar{b}_k,a_k)-(\hat{\epsilon}+4v)\hat{Q}^\pi_{M_{P_s}}(\bar{b}_k,a_k,\alpha)-(\alpha+\hat{\epsilon}+4v-1)(T-k+1)R_{max}]$
		
		\item $U \triangleq \frac{\alpha-\hat{\epsilon}+4v}{\alpha}\hat{Q}^\pi_{M_{P_s}}(\bar{b}_k,a_k,\alpha-\hat{\epsilon}) + \frac{\hat{\epsilon}}{\alpha}R_{max}(T-k+1)$
	\end{enumerate} 
	If $N_\Delta\geq -8B^2 \frac{ln(\frac{\delta/2}{4(T-k)})}{(v/(T-k))^2}$ then following hold \begin{enumerate}
		\item If $\hat{\epsilon}+\alpha<1$ then $
		P(L_1-Q_{M_P}^\pi(\bar{b}_k,a_k,\alpha)>\lambda_1+\lambda_2)\leq \delta
		$
		for $\lambda_1=-\frac{2R_{max}(T-k+1)}{\alpha}\sqrt{\frac{ln(1/(\delta/4))}{2C}}$ and $\lambda_2=\frac{\sqrt{\hat{\epsilon}}}{\alpha}2R_{max}(T-k+1)\sqrt{\frac{5ln(3/(\delta/4))}{C}}$.
		\item If $\hat{\epsilon}+\alpha\geq 1$ then $
		P(L_2-Q_{M_P}^\pi(\bar{b}_k,a_k,\alpha)>\eta_1+\eta_2)\leq \delta
		$ for $\eta_1 \triangleq \sqrt{-\frac{ln(\delta / 4)R_{max}(T-k+1)}{C^2\alpha^2}}$ and $\eta_2 \triangleq \frac{2\sqrt{\hat{\epsilon}+4v}}{\alpha}R_{max}(T-k+1)\sqrt{\frac{5ln(3/(\delta/4))}{C}}$
		\item If $\alpha>\hat{\epsilon}$ then $
		P(Q^\pi_{M_{P}}(\bar{b}_k,a_k,\alpha) - U> \lambda)\leq \delta
		$ for $\lambda=2R_{max}(T-k+1)\frac{\sqrt{\alpha-\hat{\epsilon}}}{\alpha}\sqrt{\frac{5ln(3/(\delta/2))}{C}}$
	\end{enumerate}
\end{theorem}
\begin{proof}
	First we will prove that for $\alpha_1,\alpha_2 \in (0,1)$ such that $\alpha_1\leq \alpha_2$ it holds that 
	\begin{equation}\label{eq:cvar_estimator_stochastic_monotnicity}
		\forall c\in \mathbb{R},P(\hat{Q}^\pi_{M_P}(b,a,\alpha_1)>c) \geq P(\hat{Q}^\pi_{M_P}(b,a,\alpha_2)>c).
	\end{equation} 
	Let $x_i \in \mathbb{R}$ and denote by $x^{(i)}$ the i'th number from $x_1,\dots,x_n$ in ascending order.
	$$\begin{aligned}
		\frac{d}{d\alpha}\hat{C}_\alpha(\{x_i\}_{i=1}^n)
		&=\frac{d}{d\alpha} \Bigl{[} x^{(n)}-\frac{1}{\alpha}\sum_{i=1}^n (x^{(i)} - x^{(i-1)})\Bigl{(} \frac{i}{n}-(1-\alpha)\Bigr{)}^+ \Bigr{]} \\
		&=\frac{d}{d\alpha}\Bigl{[} -\sum_{i=1}^n (x^{(i)} - x^{(i-1)})\Bigl{(} -\frac{1}{\alpha}(1-\frac{i}{n}) + 1\Bigr{)}^+ \Bigr{]} \\
		&=-\sum_{i=1}^n (x^{(i)} - x^{(i-1)})\Bigl{(} \frac{1}{\alpha^2}(1-\frac{i}{n}) \Bigr{)}1_{\frac{i}{\alpha n}-\frac{1}{\alpha}+1>0}\leq 0
	\end{aligned}$$
	and therefore $\hat{C}_\alpha (\{x_i\}_{i=1}^n)$ is a monotonic decreasing function with respect to $\alpha$. $\hat{Q}$ is a random variable, when the random components are $\bar{R}^{(i)}_{k:T}$. Let $\Omega_i$ be the sample space of $\bar{R}^{(i)}_{k:T}$ and define $\Omega=\prod_{i=1}^C \Omega_i$ to be the product space on which $\hat{Q}_{M_P}^\pi$ is defined. Let $\omega \in \Omega$ such that $\omega\triangleq (\omega_1,\dots,\omega_C)$ and note that 
	$$
		\hat{Q}_{M_P}^\pi(\bar{b},a,\alpha)(\omega)\triangleq \hat{C}(\{\bar{R}^i_{k:T}\}_{i=1}^C)(w)=\bar{R}^{(C)}_{k:T}(\omega_C)-\frac{1}{\alpha}\sum_{i=1}^C (\bar{R}^{(i)}_{k:T}(\omega_i)-\bar{R}^{(i-1)}_{k:T}(\omega_{i-1}))\Bigl(\frac{i}{n}-(1-\alpha) \Bigr)^+.
	$$
	For all $i\in \{1,\dots,C\}$, $\bar{R}_i(\omega_i)$ is constant, and because the CVaR estimator was proved to be monotonic decreasing with respect to $\alpha$ we get that if $\alpha_1\leq \alpha_2$ then $\hat{Q}_{M_P}^\pi(\bar{b},a,\alpha_1)(\omega) \geq \hat{Q}_{M_P}^\pi(\bar{b},a,\alpha_2)(\omega)$. Hence, it holds that $$
		\{w\in \Omega|\hat{Q}_{M_P}^\pi(\bar{b},a,\alpha_2)(\omega)>c\}\subset \{w\in \Omega|\hat{Q}_{M_P}^\pi(\bar{b},a,\alpha_1)(\omega)>c\}.
	$$
	for all $c\in \mathbb{R}$. Therefore,
	$$\begin{aligned}
		P(\hat{Q}^\pi_{M_P}(b,a,\alpha_1)>c) \triangleq P(\{w\in \Omega|\hat{Q}_{M_P}^\pi(\bar{b},a,\alpha_1)(\omega)>c\}) &\geq P(\{w\in \Omega|\hat{Q}_{M_P}^\pi(\bar{b},a,\alpha_2)(\omega)>c\}) \\
		&\triangleq P(\hat{Q}^\pi_{M_P}(b,a,\alpha_2)>c)
	\end{aligned}$$
	Now we know that $\hat{Q}^\pi_{M_P}(b,a,\alpha)$ is monotonic decreasing with respect to $\alpha$.
	
	From \ref{thm:m_i_sum_bound_guarantees} it holds that for $N_\Delta\geq -8B^2 \frac{ln(\frac{\delta/2}{4(T-k)})}{(v/(T-k))^2}$ we get $$P(|\sum_{i=k+1}^T \hat{m}_i-\sum_{j=k+1}^T m_j|>2v)\leq \frac{\delta}{2}.$$
	
	Denote \begin{enumerate}
		\item $
		\bar{L}_{M_{P_s}}^1(Q, \epsilon)\triangleq\frac{\alpha+\epsilon}{\alpha}Q^\pi_{M_{P_s}}(\bar{b}_k,a_k,\alpha+\epsilon)-\frac{\epsilon}{\alpha}Q_{M_{P_s}}^\pi (\bar{b}_k,a_k,\epsilon)
		$
		
		\item $
		\bar{L}_{M_{P_s}}^2(Q,\epsilon)\triangleq\frac{1}{\alpha}[Q^\pi_{M_{P_s}}(\bar{b}_k,a_k)-\epsilon 
		Q^\pi_{M_{P_s}}(\bar{b}_k,a_k,\alpha)-(\alpha+\epsilon-1)(T-k+1)R_{max}]
		$
		
		\item $\bar{U}_{M_{P_s}}(Q,\epsilon)\triangleq\frac{\alpha-\epsilon}{\alpha}Q_{M_{P_s}}(\bar{b}_k,a_k,\alpha-\epsilon)+\frac{\epsilon}{\alpha}R_{max}(T-k+1)$
		
		\item $\tilde{\lambda}_2=\frac{\sqrt{\tilde{\epsilon}}}{\alpha}2R_{max}(T-k+1)\sqrt{\frac{5ln(3/(\delta/4))}{ C}}$
		
		\item $\tilde{\epsilon} \triangleq \epsilon - 2v$ and $\epsilon^+ \triangleq \epsilon + 2v$
		
		\item $\tilde{\lambda}=2R_{max}(T-k+1)\frac{\sqrt{\alpha-\tilde{\epsilon}}}{\alpha}\sqrt{\frac{5ln(3/(\delta/2))}{C}}$
		
		\item $\eta \triangleq \frac{2\sqrt{\epsilon^+}}{\alpha}R_{max}(T-k+1)\sqrt{\frac{5ln(3/(\delta/4))}{C}}$
	\end{enumerate}
	
	If $\hat{\epsilon}<1-\alpha$ then
	$$\begin{aligned}
		&P(\frac{\alpha+\hat{\epsilon}-4v}{\alpha}\hat{Q}^\pi_{M_{P_s}}(\bar{b}_k,a_k,\alpha+\hat{\epsilon}) - \frac{\hat{\epsilon}}{\alpha}\hat{Q}^\pi_{M_{P_s}}(\bar{b}_k,a_k,\hat{\epsilon}-4v) -Q^\pi_{M_{P}}(\bar{b}_k,a_k,\alpha)> \lambda_1 + \lambda_2) \\
		&=P(\frac{\alpha+\hat{\epsilon}-4v}{\alpha}\hat{Q}^\pi_{M_{P_s}}(\bar{b}_k,a_k,\alpha+\hat{\epsilon}) - \frac{\hat{\epsilon}}{\alpha}\hat{Q}^\pi_{M_{P_s}}(\bar{b}_k,a_k,\hat{\epsilon}-4v) -Q^\pi_{M_{P}}(\bar{b}_k,a_k,\alpha) \\
			&> \lambda_1 + \lambda_2\Bigr{|} |\epsilon - \hat{\epsilon}|>2v) 
		\underbrace{P(|\epsilon - \hat{\epsilon}|>2v)}_{\leq \delta/2} \\
		&+P(\frac{\alpha+\hat{\epsilon}-4v}{\alpha}\hat{Q}^\pi_{M_{P_s}}(\bar{b}_k,a_k,\alpha+\hat{\epsilon}) - \frac{\hat{\epsilon}}{\alpha}\hat{Q}^\pi_{M_{P_s}}(\bar{b}_k,a_k,\hat{\epsilon}-4v) -Q^\pi_{M_{P}}(\bar{b}_k,a_k,\alpha) \\
		&> \lambda_1 + \lambda_2\Bigr{|} |\epsilon - \hat{\epsilon}|\leq 2v)
		\times \underbrace{P(|\epsilon - \hat{\epsilon}|\leq 2v)}_{\leq 1} \\
		&\overset{1}{\leq} \frac{\delta}{2} + P(\frac{\alpha+\hat{\epsilon}-4v}{\alpha}\hat{Q}^\pi_{M_{P_s}}(\bar{b}_k,a_k,\alpha+\hat{\epsilon}) - \frac{\hat{\epsilon}}{\alpha}\hat{Q}^\pi_{M_{P_s}}(\bar{b}_k,a_k,\hat{\epsilon}-4v) - Q^\pi_{M_{P}}(\bar{b}_k,a_k,\alpha) \\
		&> \lambda_1 + \lambda_2\Bigr{|} |\epsilon - \hat{\epsilon}|\leq 2v) \\
		&\overset{2}{\leq} \frac{\delta}{2} + P(\frac{\alpha+\epsilon+2v-4v}{\alpha}\hat{Q}^\pi_{M_{P_s}}(\bar{b}_k,a_k,\alpha+\hat{\epsilon}) - \frac{\hat{\epsilon}}{\alpha}\hat{Q}^\pi_{M_{P_s}}(\bar{b}_k,a_k,\epsilon+2v-4v) -Q^\pi_{M_{P}}(\bar{b}_k,a_k,\alpha) \\
		&> \lambda_1 + \lambda_2\Bigr{|} |\epsilon - \hat{\epsilon}|\leq 2v) \\
		&= \frac{\delta}{2} + P(\frac{\alpha+\tilde{\epsilon}}{\alpha}\hat{Q}^\pi_{M_{P_s}}(\bar{b}_k,a_k,\alpha+\hat{\epsilon}) - \frac{\hat{\epsilon}}{\alpha}\hat{Q}^\pi_{M_{P_s}}(\bar{b}_k,a_k,\tilde{\epsilon}) -Q^\pi_{M_{P}}(\bar{b}_k,a_k,\alpha)> \lambda_1 + \lambda_2\Bigr{|} |\epsilon - \hat{\epsilon}|\leq 2v) \\
		&\overset{2}{\leq} \frac{\delta}{2} + P(\frac{\alpha+\tilde{\epsilon}}{\alpha}\hat{Q}^\pi_{M_{P_s}}(\bar{b}_k,a_k,\alpha+\tilde{\epsilon}) - \frac{\tilde{\epsilon}}{\alpha}\hat{Q}^\pi_{M_{P_s}}(\bar{b}_k,a_k,\tilde{\epsilon}) -Q^\pi_{M_{P}}(\bar{b}_k,a_k,\alpha)> \lambda_1 + \lambda_2\Bigr{|} |\epsilon - \hat{\epsilon}|\leq 2v) \\
		&=\frac{\delta}{2} + P(\frac{\alpha+\tilde{\epsilon}}{\alpha}\hat{Q}^\pi_{M_{P_s}}(\bar{b}_k,a_k,\alpha+\tilde{\epsilon}) - \frac{\tilde{\epsilon}}{\alpha}\hat{Q}^\pi_{M_{P_s}}(\bar{b}_k,a_k,\tilde{\epsilon}) + \bar{L}^1_{M_{P_s}}(\hat{Q},\tilde{\epsilon}) \\
		&- \bar{L}^1_{M_{P_s}}(\hat{Q},\tilde{\epsilon})-Q^\pi_{M_{P}}(\bar{b}_k,a_k,\alpha)> \lambda_1 + \lambda_2\Bigr{|} |\epsilon - \hat{\epsilon}|\leq 2v) \\
		&\overset{3}{\leq} \frac{\delta}{2} + P(\frac{\alpha+\tilde{\epsilon}}{\alpha}\hat{Q}^\pi_{M_{P_s}}(\bar{b}_k,a_k,\alpha+\tilde{\epsilon}) - \frac{\tilde{\epsilon}}{\alpha}\hat{Q}^\pi_{M_{P_s}}(\bar{b}_k,a_k,\tilde{\epsilon}) - \bar{L}^1_{M_{P_s}}(\hat{Q},\tilde{\epsilon})>0 \Bigr{|} |\epsilon - \hat{\epsilon}|\leq 2v) \\
		&+ P(\bar{L}^1_{M_{P_s}}(\hat{Q},\tilde{\epsilon})-Q^\pi_{M_{P}}(\bar{b}_k,a_k,\alpha)> \lambda_1 + \lambda_2\Bigr{|} |\epsilon - \hat{\epsilon}|\leq 2v) \\
		&\overset{4}{\leq} \frac{\delta}{2} + P(0>0\Bigr{|} |\epsilon - \hat{\epsilon}|\leq 2v) +  P(\bar{L}^1_{M_{P_s}}(\hat{Q},\tilde{\epsilon})-Q^\pi_{M_{P}}(\bar{b}_k,a_k,\alpha)> \lambda_1 + \tilde{\lambda}_2 \Bigr{|} |\epsilon - \hat{\epsilon}|\leq 2v) \\
		&\overset{5}{\leq} \delta
	\end{aligned}$$
	
	If $\hat{\epsilon}\geq 1-\alpha$ then
	$$\begin{aligned}
		&P(\frac{1}{\alpha}[\hat{Q}^\pi_{M_{P_s}}(\bar{b}_k,a_k)-(\hat{\epsilon}+4v)\hat{Q}^\pi_{M_{P_s}}(\bar{b}_k,a_k,\alpha)-(\alpha+\hat{\epsilon}+4v-1)(T-k+1)R_{max}] \\
		&- Q_{M_{P}}^\pi(\bar{b}_k,a_k,\alpha)>\eta_1 + \eta_2) \\
		&= P(\frac{1}{\alpha}[\hat{Q}^\pi_{M_{P_s}}(\bar{b}_k,a_k)-(\hat{\epsilon}+4v)\hat{Q}^\pi_{M_{P_s}}(\bar{b}_k,a_k,\alpha)-(\alpha+\hat{\epsilon}+4v-1)(T-k+1)R_{max}] \\
		&- Q_{M_{P}}^\pi(\bar{b}_k,a_k,\alpha)>\eta_1 + \eta_2 \Bigr{|} |\epsilon - \hat{\epsilon}| > 2v)\underbrace{P(|\epsilon - \hat{\epsilon}| > 2v)}_{\leq \delta/2} \\
		&+P(\frac{1}{\alpha}[\hat{Q}^\pi_{M_{P_s}}(\bar{b}_k,a_k)-(\hat{\epsilon}+4v)\hat{Q}^\pi_{M_{P_s}}(\bar{b}_k,a_k,\alpha)-(\alpha+\hat{\epsilon}+4v-1)(T-k+1)R_{max}] \\
		&- Q_{M_{P}}^\pi(\bar{b}_k,a_k,\alpha)>\eta_1 + \eta_2 \Bigr{|} |\epsilon - \hat{\epsilon}| \leq 2v)\underbrace{P(|\epsilon - \hat{\epsilon}| \leq 2v)}_{\leq 1} \\
		&\overset{1}{\leq} \frac{\delta}{2} + P(\frac{1}{\alpha}[\hat{Q}^\pi_{M_{P_s}}(\bar{b}_k,a_k)-(\hat{\epsilon}+4v)\hat{Q}^\pi_{M_{P_s}}(\bar{b}_k,a_k,\alpha)-(\alpha+\hat{\epsilon}+4v-1)(T-k+1)R_{max}] \\
		&- Q_{M_{P}}^\pi(\bar{b}_k,a_k,\alpha)>\eta_1 + \eta_2 \Bigr{|} |\epsilon - \hat{\epsilon}| \leq 2v) \\
		&\overset{8}{\leq} \frac{\delta}{2} + P(\frac{1}{\alpha}[\hat{Q}^\pi_{M_{P_s}}(\bar{b}_k,a_k)-\epsilon^+ \hat{Q}^\pi_{M_{P_s}}(\bar{b}_k,a_k,\alpha)-(\alpha+\epsilon^+-1)(T-k+1)R_{max}] \\
		&- Q_{M_{P}}^\pi(\bar{b}_k,a_k,\alpha)>\eta_1 + \eta_2 \Bigr{|} |\epsilon - \hat{\epsilon}| \leq 2v) \\
		&\overset{3}{\leq} \frac{\delta}{2} + P(\frac{1}{\alpha}[\hat{Q}^\pi_{M_{P_s}}(\bar{b}_k,a_k)-\epsilon^+ \hat{Q}^\pi_{M_{P_s}}(\bar{b}_k,a_k,\alpha)-(\alpha+\epsilon^+-1)(T-k+1)R_{max}] \\
		&- \bar{L}_{M_{P_s}}(\hat{Q}, \epsilon^+) > 0 \Bigr{|} |\epsilon - \hat{\epsilon}| \leq 2v) \\
		&+ P(\bar{L}_{M_{P_s}}(\hat{Q}, \epsilon^+) - Q_{M_{P}}^\pi(\bar{b}_k,a_k,\alpha)>\eta_1 + \eta_2 \Bigr{|} |\epsilon - \hat{\epsilon}| \leq 2v) \\
		&= \frac{\delta}{2} + P(0 > 0 \Bigr{|} |\epsilon - \hat{\epsilon}| \leq 2v) 
		+ P(\bar{L}_{M_{P_s}}(\hat{Q}, \epsilon^+) - Q_{M_{P}}^\pi(\bar{b}_k,a_k,\alpha)>\lambda \Bigr{|} |\epsilon - \hat{\epsilon}| \leq 2v) \\
		&=\frac{\delta}{2} + P(\bar{L}_{M_{P_s}}(\hat{Q}, \epsilon^+) - Q_{M_{P}}^\pi(\bar{b}_k,a_k,\alpha)>\eta_1 + \eta_2 \Bigr{|} |\epsilon - \hat{\epsilon}| \leq 2v) \\
		&\overset{8}{\leq} \frac{\delta}{2} + P(\bar{L}_{M_{P_s}}(\hat{Q}, \epsilon^+) - Q_{M_{P}}^\pi(\bar{b}_k,a_k,\alpha)>\eta_1+\eta \Bigr{|} |\epsilon - \hat{\epsilon}| \leq 2v) \\
		&\overset{9}{\leq} \delta
	\end{aligned}$$
	
	If $\alpha>\hat{\epsilon}$ then
	$$\begin{aligned}
		&P(Q^\pi_{M_{P}}(\bar{b}_k,a_k,\alpha) - \frac{\alpha-\hat{\epsilon}+4v}{\alpha}\hat{Q}^\pi_{M_{P_s}}(\bar{b}_k,a_k,\alpha-\hat{\epsilon}) - \frac{\hat{\epsilon}}{\alpha}R_{max}(T-k+1)> \lambda) \\
		&=P(Q^\pi_{M_{P}}(\bar{b}_k,a_k,\alpha) - \frac{\alpha-\hat{\epsilon}+4v}{\alpha}\hat{Q}^\pi_{M_{P_s}}(\bar{b}_k,a_k,\alpha-\hat{\epsilon}) - \frac{\hat{\epsilon}}{\alpha}R_{max}(T-k+1) \\
		&> \lambda \Bigr{|} |\epsilon - \hat{\epsilon}|\leq 2v)\underbrace{P(|\epsilon - \hat{\epsilon}|\leq 2v)}_{\leq 1} \\
		&+P(Q^\pi_{M_{P}}(\bar{b}_k,a_k,\alpha) - \frac{\alpha-\hat{\epsilon}+4v}{\alpha}\hat{Q}^\pi_{M_{P_s}}(\bar{b}_k,a_k,\alpha-\hat{\epsilon}) - \frac{\hat{\epsilon}}{\alpha}R_{max}(T-k+1) \\
			&> \lambda \Bigr{|} |\epsilon - \hat{\epsilon}| > 2v) \underbrace{P(|\epsilon - \hat{\epsilon}| > 2v)}_{\leq \delta/2} \\
		&\overset{1}{\leq} \frac{\delta}{2} + P(Q^\pi_{M_{P}}(\bar{b}_k,a_k,\alpha) - \frac{\alpha-\hat{\epsilon}+4v}{\alpha}\hat{Q}^\pi_{M_{P_s}}(\bar{b}_k,a_k,\alpha-\hat{\epsilon}) - \frac{\hat{\epsilon}}{\alpha}R_{max}(T-k+1) \\
		&> \lambda \Bigr{|} |\epsilon - \hat{\epsilon}|\leq 2v) \\
		&\overset{3}{\leq} \frac{\delta}{2} + P(Q^\pi_{M_{P}}(\bar{b}_k,a_k,\alpha) - \bar{U}_{M_{P_s}}(\hat{Q},\tilde{\epsilon}) > \lambda \Bigr{|} |\epsilon - \hat{\epsilon}|\leq 2v) \\
		&+ P(\bar{U}_{M_{P_s}}(\hat{Q},\tilde{\epsilon}) - \frac{\alpha-\hat{\epsilon}+4v}{\alpha}\hat{Q}^\pi_{M_{P_s}}(\bar{b}_k,a_k,\alpha-\hat{\epsilon}) - \frac{\hat{\epsilon}}{\alpha}R_{max}(T-k+1)> 0 \Bigr{|} |\epsilon - \hat{\epsilon}|\leq 2v) \\
		&\overset{}{\leq} \frac{\delta}{2} + P(Q^\pi_{M_{P}}(\bar{b}_k,a_k,\alpha) - \bar{U}_{M_{P_s}}(\hat{Q},\tilde{\epsilon}) > \tilde{\lambda} \Bigr{|} |\epsilon - \hat{\epsilon}|\leq 2v) \\
		&+ P(\bar{U}_{M_{P_s}}(\hat{Q},\tilde{\epsilon}) - \frac{\alpha-\hat{\epsilon}+4v}{\alpha}\hat{Q}^\pi_{M_{P_s}}(\bar{b}_k,a_k,\alpha-\hat{\epsilon}) - \frac{\hat{\epsilon}}{\alpha}R_{max}(T-k+1)> 0 \Bigr{|} |\epsilon - \hat{\epsilon}|\leq 2v) \\
		&\overset{6}{\leq} \delta + P(\bar{U}_{M_{P_s}}(\hat{Q},\tilde{\epsilon}) - \frac{\alpha-\hat{\epsilon}+4v}{\alpha}\hat{Q}^\pi_{M_{P_s}}(\bar{b}_k,a_k,\alpha-\hat{\epsilon}) - \frac{\hat{\epsilon}}{\alpha}R_{max}(T-k+1)> 0 \Bigr{|} |\epsilon - \hat{\epsilon}|\leq 2v) \\
		&\overset{7}{\leq} \delta + P(\bar{U}_{M_{P_s}}(\hat{Q},\tilde{\epsilon}) - \frac{\alpha-\tilde{\epsilon}}{\alpha}\hat{Q}^\pi_{M_{P_s}}(\bar{b}_k,a_k,\alpha-\tilde{\epsilon}) - \frac{\tilde{\epsilon}}{\alpha}R_{max}(T-k+1)> 0 \Bigr{|} |\epsilon - \hat{\epsilon}|\leq 2v) \\
		&\leq \delta + P(0>0 \Bigr{|} |\epsilon - \hat{\epsilon}|\leq 2v) \\
		&= \delta
	\end{aligned}$$
	
	$^1$ Theorem \ref{thm:m_i_sum_bound_guarantees}, which is stated in the beginning of the proof.\\
	$^2$ $v>0$ and $\hat{Q}$ is monotonic decreasing with respect to $\alpha$. \\
	$^3$ Theorem \ref{thm:prob_x_plus_y_bound}.\\
	$^4$ Given $|\hat{\epsilon}-\epsilon|\leq 2v$ it holds that $\tilde{\lambda}_2\geq \lambda_2$. \\
	$^5$ Given $|\hat{\epsilon}-\epsilon|\leq 2v$ it holds that $\tilde{\epsilon}\leq 1-\alpha$. Hence, by applying Theorem \ref{thm:unfirom_bound_convergence_guarantees_with_constant_epsilon} over $\tilde{\epsilon}$ we get the inequality. \\
	$^6$ Given $|\hat{\epsilon}-\epsilon|\leq 2v$ and $\alpha>\hat{\epsilon}$ it holds that $\tilde{\epsilon}<\alpha$. Hence, by applying Theorem \ref{thm:unfirom_bound_convergence_guarantees_with_constant_epsilon} over $\tilde{\epsilon}$ we get the inequality. \\
	$^7$ $\tilde{\epsilon}\leq \hat{\epsilon} \leq \epsilon + 2v$. \\
	$^8$ Given $|\epsilon - \hat{\epsilon}|\leq 2v$ we get $\hat{\epsilon}+4v \geq \epsilon -2v + 4v=\epsilon^+$. \\
	$^9$ Given $|\hat{\epsilon}-\epsilon|\leq 2v$ we get $\epsilon^+ + \alpha \geq 1$ and therefore the inequality holds from Theorem \ref{thm:unfirom_bound_convergence_guarantees_with_constant_epsilon}.
\end{proof}

\begin{theorem}\label{thm:unfirom_bound_convergence_guarantees_with_constant_epsilon}
	Let $\delta\in (0,1), \alpha\in (0,1), \lambda>0$, and denote
	
	\begin{enumerate}
		\item $
		\epsilon\triangleq\sum_{i=k+1}^{T-1} E_{\bar{b}_{k+1:i}}[\Delta^s(\bar{b}_i,a_i)|\bar{b}_k,\pi]
		$ 
		\item $
		\bar{L}_{M_{P_s}}^1(Q)\triangleq\frac{\alpha+\epsilon}{\alpha}Q^\pi_{M_{P_s}}(\bar{b}_k,a_k,\alpha+\epsilon)-\frac{\epsilon}{\alpha}Q_{M_{P_s}}^\pi (\bar{b}_k,a_k,\epsilon)
		$
		
		\item $
		\bar{L}_{M_{P_s}}^2(Q)\triangleq\frac{1}{\alpha}[Q^\pi_{M_{P_s}}(\bar{b}_k,a_k)-\epsilon 
		Q^\pi_{M_{P_s}}(\bar{b}_k,a_k,\alpha)-(\alpha+\epsilon-1)(T-k+1)R_{max}]
		$
		
		\item $\bar{U}_{M_{P_s}}(Q)\triangleq\frac{\alpha-\epsilon}{\alpha}Q_{M_{P_s}}(\bar{b}_k,a_k,\alpha-\epsilon)+\frac{\epsilon}{\alpha}R_{max}(T-k+1)$
	\end{enumerate}
	The following bounds hold \begin{enumerate}
		\item If $\epsilon+\alpha<1$ then
		$$
		P(\bar{L}_{M_{P_s}}^1(\hat{Q})-Q_{M_{P}}(\bar{b}_k,a_k,\alpha)>\lambda_1+\lambda_2)\leq \delta
		$$
		for $\lambda_1=-\frac{2R_{max}(T-k+1)}{\alpha}\sqrt{\frac{ln(1/(\delta/2))}{2C}}$ and $\lambda_2=\frac{\epsilon}{\alpha}2R_{max}(T-k+1)\sqrt{\frac{5ln(3/(\delta/2))}{\epsilon C}}$.
		
		\item If $\epsilon+\alpha\geq 1$ then 
		$$
		P(\bar{L}_{M_{P_s}}^2(\hat{Q})-Q_{M_{P}}(\bar{b}_k,a_k,\alpha)>\lambda_1+\lambda_2)\leq \delta
		$$
		for $\lambda_1=\sqrt{-\frac{ln(\delta/2)R_{max}(T-k+1)}{C^2 \alpha^2}}$ and $\lambda_2=\frac{\epsilon}{\alpha}2R_{max}(T-k+1)\sqrt{\frac{5ln(3/(\delta/2))}{\epsilon C}}$.
		
		\item If $\alpha>\epsilon$ then $$
		P(Q_{M_{P_s}}(\bar{b}_k,a_k,\alpha)-\bar{U}_{M_{P}}(\hat{Q})>\lambda)\leq \delta
		$$
		for $\lambda=2R_{max}(T-k+1)\frac{\alpha-\epsilon}{\alpha}\sqrt{\frac{5ln(3/\delta)}{(\alpha-\epsilon)C}}$.
	\end{enumerate}
\end{theorem}
\begin{proof}\label{prf:unfirom_bound_convergence_guarantees_with_constant_epsilon}
	If $\epsilon+\alpha<1$ then
	$$
	\begin{aligned}
		&P(\bar{L}^1_{M_{P_s}}(\hat{Q})-Q_{M_P}(\bar{b}_k,a_k,\alpha)>\lambda_1+\lambda_2)=P(\bar{L}^1_{M_{P_s}}(\hat{Q})-\bar{L}^1_{M_{P_s}}(Q)+\bar{L}^1_{M_{P_s}}(Q) \\
		&-Q_{M_P}(\bar{b}_k,a_k,\alpha)>\lambda_1+\lambda_2) \\
		&\overset{1}{\leq} P(\bar{L}^1_{M_{P_s}}(\hat{Q})-\bar{L}^1_{M_{P_s}}(Q)>\lambda_1+\lambda_2) + P(\bar{L}^1_{M_{P_s}}(Q)-Q_{M_P}(\bar{b}_k,a_k,\alpha)>0) \\
		&\overset{2}{=}P(\bar{L}^1_{M_{P_s}}(\hat{Q})-\bar{L}^1_{M_{P_s}}(Q)>\lambda_1+\lambda_2) \\
		&=P(\frac{\alpha+\epsilon}{\alpha}(\hat{Q}^\pi_{P_s}(\bar{b}_k,a_k,\alpha+\epsilon)-Q^\pi_{P_s}(\bar{b}_k,a_k,\alpha+\epsilon))-\frac{\epsilon}{\alpha}(\hat{Q}_{P_s}^\pi (\bar{b}_k,a_k,\epsilon)-Q_{P_s}^\pi (\bar{b}_k,a_k,\epsilon))>\lambda_1+\lambda_2) \\
		&\overset{1}{\leq} P(\hat{Q}^\pi_{P_s}(\bar{b}_k,a_k,\alpha+\epsilon)-Q^\pi_{P_s}(\bar{b}_k,a_k,\alpha+\epsilon)>\lambda_1\frac{\alpha}{\alpha+\epsilon}) + P(\hat{Q}^\pi_{P_s}(\bar{b}_k,a_k,\epsilon)-Q^\pi_{P_s}(\bar{b}_k,a_k,\epsilon)<-\lambda_2\frac{\alpha}{\epsilon}) \\
		&=P(Q^\pi_{P_s}(\bar{b}_k,a_k,\alpha+\epsilon)-\hat{Q}^\pi_{P_s}(\bar{b}_k,a_k,\alpha+\epsilon)<-\lambda_1\frac{\alpha}{\alpha+\epsilon}) + P(Q^\pi_{P_s}(\bar{b}_k,a_k,\epsilon)-\hat{Q}^\pi_{P_s}(\bar{b}_k,a_k,\epsilon)>\lambda_2\frac{\alpha}{\epsilon}) \\
		&=P(Q^\pi_{P_s}(\bar{b}_k,a_k,\alpha+\epsilon)-\hat{Q}^\pi_{P_s}(\bar{b}_k,a_k,\alpha+\epsilon)<\frac{2R_{max}(T-k+1)}{\alpha+\epsilon}\sqrt{\frac{ln(1/(\delta/2))}{2C}}) \\
		&+ P(Q^\pi_{P_s}(\bar{b}_k,a_k,\epsilon)-\hat{Q}^\pi_{P_s}(\bar{b}_k,a_k,\epsilon)>2R_{max}(T-k+1)\sqrt{\frac{5ln(3/(\delta/2))}{\epsilon C}}) \\
		&\overset{3}{\leq} \frac{\delta}{2} + \frac{\delta}{2} = \delta
	\end{aligned}
	$$
	$^1$ Follows from Theorem \ref{thm:prob_x_plus_y_bound}. \\
	$^2$ In Theorem \ref{thm:uniform_lower_and_upper_bounds_for_v_and_q} we showed that $\bar{L}^1_{M_{P_s}}$ is the lower bound of $Q_{M_P}(\bar{b}_k,a_km\alpha)$ and therefore $P(\bar{L}^1_{M_{P_s}}(Q)-Q_{M_P}(\bar{b}_k,a_k,\alpha)>0)=0$.\\
	$^3$ Follows from \ref{thm:brown_bounds}. \\
	
	If $\epsilon+\alpha\geq 1$ then 
	$$
	\begin{aligned}
		&P(\bar{L}^2_{M_{P_s}}(\hat{Q})-Q_{M_P}(\bar{b}_k,a_k,\alpha)>\lambda_1+\lambda_2)=P(\bar{L}^2_{M_{P_s}}(\hat{Q})-\bar{L}^2_{M_{P_s}}(Q)+\bar{L}^2_{M_{P_s}}(Q) \\
		&-Q_{M_P}(\bar{b}_k,a_k,\alpha)>\lambda_1+\lambda_2) \\
		&\overset{1}{\leq} P(\bar{L}^2_{M_{P_s}}(\hat{Q})-\bar{L}^2_{M_{P_s}}(Q)>\lambda_1+\lambda_2) + P(\bar{L}^2_{M_{P_s}}(Q)-Q_{M_P}(\bar{b}_k,a_k,\alpha)>0) \\
		&\overset{4}{=}P(\bar{L}^2_{M_{P_s}}(\hat{Q})-\bar{L}^2_{M_{P_s}}(Q)>\lambda_1+\lambda_2) \\
		&=P(\frac{1}{\alpha}(\hat{Q}^\pi_{P_s}(\bar{b}_k,a_k)-Q^\pi_{P_s}(\bar{b}_k,a_k))-\epsilon(\hat{Q}_{P_s}^\pi (\bar{b}_k,a_k,\epsilon)-Q_{P_s}^\pi (\bar{b}_k,a_k,\epsilon))>\lambda_1+\lambda_2) \\
		&\overset{1}{\leq} P(\hat{Q}^\pi_{P_s}(\bar{b}_k,a_k)-Q^\pi_{P_s}(\bar{b}_k,a_k)>\lambda_1\alpha) + P(\hat{Q}^\pi_{P_s}(\bar{b}_k,a_k,\epsilon)-Q^\pi_{P_s}(\bar{b}_k,a_k,\epsilon)<-\lambda_2\epsilon) \\
		&=P(\frac{1}{C}\sum_{i=1}^C R_{k:T}^i-E[R_{k:T}|b_k,a_k,\pi]>\lambda_1\alpha) + P(Q^\pi_{P_s}(\bar{b}_k,a_k,\epsilon)-\hat{Q}^\pi_{P_s}(\bar{b}_k,a_k,\epsilon)>\lambda_2\epsilon) \\
		&=\underbrace{P(\frac{1}{C}\sum_{i=1}^C R_{k:T}^i-E[R_{k:T}|b_k,a_k,\pi]>\sqrt{-\frac{ln(\delta/2)R_{max}(T-k+1)}{C^2}})}_{A} \\
		&+ \underbrace{P(Q^\pi_{P_s}(\bar{b}_k,a_k,\epsilon)-\hat{Q}^\pi_{P_s}(\bar{b}_k,a_k,\epsilon)>2R_{max}(T-k+1)\sqrt{\frac{5ln(3/(\delta/2))}{\epsilon C}})}_{B} \\
		&\overset{5}{\leq} \frac{\delta}{2} + \frac{\delta}{2} = \delta
	\end{aligned}
	$$
	$^4$ In Theorem \ref{thm:uniform_lower_and_upper_bounds_for_v_and_q} we showed that $\bar{L}^2_{M_{P_s}}$ is the lower bound of $Q_{M_P}(\bar{b}_k,a_km\alpha)$ and therefore $P(\bar{L}^1_{M_{P_s}}(Q)-Q_{M_P}(\bar{b}_k,a_k,\alpha)>0)=0$.\\
	$^5$ $B\leq \delta/2$ from Theorem \ref{thm:brown_bounds} and A from Hoeffding's bound.
	$$
	\begin{aligned}
		A &= P(\sum_{i=1}^C R_{k:T}^i-E[CR_{k:T}|b_k,a_k,\pi]>\sqrt{-ln(\delta/2)R_{max}(T-k+1)}) \\
		&\leq exp\{\frac{-2(-ln(\delta/2)R_{max}(T-k+1))}{2R_{max}(T-k+1)}\}=\delta/2.
	\end{aligned}
	$$
	If $\alpha>\epsilon$ then
	$$
	\begin{aligned}
		&P(Q_{M_{P_s}}(\bar{b}_k,a_k,\alpha)-\bar{U}_{M_{P_s}}(\hat{Q})>\lambda)= P(Q_{M_{P_s}}(\bar{b}_k,a_k,\alpha)-\bar{U}_{M_{P_s}}(Q)+\bar{U}_{M_{P_s}}(Q)-\bar{U}_{M_{P_s}}(\hat{Q})>\lambda)\\
		&\overset{1}{\leq} P(Q_{M_{P_s}}(\bar{b}_k,a_k,\alpha)-\bar{U}_{M_{P_s}}(Q)>0) + P(Q_{M_{P_s}}(\bar{U}_{M_{P_s}}(Q)-\bar{U}_{M_{P_s}}(\hat{Q})>\lambda) \\
		&\overset{6}{=}P(\bar{U}_{M_{P_s}}(Q)-\bar{U}_{M_{P_s}}(\hat{Q})>\lambda) \\
		&=P(\frac{\alpha-\epsilon}{\alpha}(Q^\pi_{M_{P_s}}(b_k,a_k,\alpha-\epsilon)- \hat{Q}^\pi_{M_{P_s}}(b_k,a_k,\alpha-\epsilon))>\lambda) \\
		&=P(Q^\pi_{M_{P_s}}(b_k,a_k,\alpha-\epsilon)- \hat{Q}^\pi_{M_{P_s}}(b_k,a_k,\alpha-\epsilon)>2R_{max}(T-k+1)\sqrt{\frac{5ln(3/\delta)}{(\alpha-\epsilon)C}}) \\
		&\overset{3}{\leq} \delta
	\end{aligned}
	$$
	$^6$ Theorem \ref{thm:uniform_lower_and_upper_bounds_for_v_and_q} shows that $Q_{M_{P_s}}(\bar{b}_k,a_k,\alpha)\leq\bar{U}_{M_{P_s}}(Q)$.
\end{proof}

\begin{theorem}\label{prf:tight_lower_bound_guarantees_estimated_g}
	Let $\eta>0,\delta \in (0,1),\alpha \in (0,1)$. 
	If $N_\Delta \geq -ln(\frac{((\delta/4)/I)/(T-1-k)}{2})\frac{2B^2}{\eta^2 / (T-1-k)^2}$, then 
	$$P(\hat{Q}_{M_{P_s}}^{\pi,\hat{h}^+ +
		\eta}(\bar{b}_k,a_k,\alpha) - Q_{M_P}^\pi(\bar{b}_k,a_k,\alpha)>v|\bar{b}_k,\pi)\leq \delta$$
	for $v=\frac{2R_{max}(T-k+1)}{\alpha}\sqrt{\frac{ln(1/(\delta/4))}{2N_\Delta}}$ and $B_i \triangleq \sup_{b_i} \frac{P(b_i|\bar{b}_k, \pi)}{Q_0{b_i}}, B \triangleq \max_{i\in \{k+1,\dots,T-1\}} B_i$.
\end{theorem}
\begin{proof}
	First, we construct the theoretical bound from Theorem \ref{thm:tight_lower_bound_for_v_and_q}, where the theoretical return $R_{k:T}$ is replaced by the particle belief return $\bar{R}_{k:T}$. Let g be the function from Theorem \ref{thm:tight_lower_bound_for_v_and_q} and denote
	$$
	P(\bar{R}^{Y^L}_{g}\leq l|\bar{b}_k,a_k,\pi)\triangleq min(1, P_{M_{P_s}}(\bar{R}_{k:T}\leq l|\bar{b}_k,a_k,\pi) + g(l)).
	$$
	Let $\bar{R}^{Y^L}_g \sim P(\bar{R}^{Y^L}_{g}\leq l|\bar{b}_k,a_k,\pi)$ be a theoretical random variable, and $\bar{R}^{Y^L}_{i,g} \overset{iid}{\sim} P(\bar{R}^{Y^L}_{g}\leq l|\bar{b}_k,a_k,\pi)$ a sample of size $N_\Delta$. The lower bound estimator using g is estimated by $\hat{Q}^{\pi,g}_{M_{P_s}}(\bar{b}_k,a_k,\alpha) \triangleq \hat{C}_\alpha(\{\bar{R}^{Y^L}_{i,g}\}_{i=1}^{N_\Delta}).$

	Recall that $P(\sup_{x\in \mathbb{R}}\{g(x)-\hat{h}^+(x)\}>\eta|\bar{b}_k,a_k,\pi)\leq \frac{\delta}{4}$ holds directly from Theorem \ref{thm:h_guarantees}, and denote $\lambda=2R_{max}(T-k+1)\sqrt{\frac{5ln(3/(\delta/4))}{N_\Delta\alpha}}$. From Theorem \ref{thm:prob_x_plus_y_bound} we get
	$$\begin{aligned}
		&P(\hat{Q}_{M_{P_s}}^{\pi,\hat{h}^+ +\eta}(\bar{b}_k,a_k,\alpha) - Q_{M_P}^\pi(\bar{b}_k,a_k,\alpha)>v|\bar{b}_k,a_k,\pi) \leq 
		\underbrace{P(\hat{Q}_{M_{P_s}}^{\pi,\hat{h}^+ +\eta}(\bar{b}_k,a_k,\alpha) - CVaR_\alpha (\bar{R}_{\hat{h}^+ +\eta}^{Y^L})>-\lambda|\bar{b}_k,\pi)}_{A_1} \\
		&+\underbrace{P(CVaR_\alpha (\bar{R}_{\hat{h}^+ +\eta}^{Y^L}) - CVaR_\alpha (\bar{R}^{Y^L}_{g}) > 0 |\bar{b}_k,a_k,\pi)}_{A_2} + \underbrace{P(CVaR_\alpha (\bar{R}^{Y^L}_{g}) - \hat{C}_\alpha(\{\bar{R}^{Y^L}_{i,g}\})>\lambda|\bar{b}_k,a_k,\pi)}_{A_3} \\
		&+\underbrace{P(\hat{C}_\alpha(\{\bar{R}^{Y^L}_{i,g}\}) - Q_{M_P}(\bar{b}_k,a_k,\alpha) > v|\bar{b}_k,a_k,\pi)}_{A_4}
	\end{aligned}$$
	From Theorem \ref{thm:brown_bounds} we get that $A_1\leq \frac{\delta}{4}$ and $A_3\leq \frac{\delta}{4}$, and from Theorem \ref{thm:tight_lower_bound_guarantees_theoretical_g} we get that $A_4\leq \frac{\delta}{4}$. Assume that $\forall x\in \mathbb{R},g(x)-\hat{h}^+(x)\leq \eta$ and let $x\in supp(R_{k:T})$. There exists i such that $x\in (k_{i-1}, k_i]$.
	$$\begin{aligned}
		&P(\bar{R}^{Y^L}_{\hat{h}^++\eta}\leq x)=\min(P(\bar{R}_{k:T}\leq x)+\hat{h}^+(x)+\eta,1) \geq \min(P(\bar{R}_{k:T}\leq x)+g(x),1)=P(\bar{R}^{Y^L}_{g} \leq x).
	\end{aligned}$$
	We got that $\bar{R}^{Y^L}_{\hat{h}^++\eta} \leq \bar{R}^{Y^L}_{g}$, so $CVaR_\alpha(\bar{R}^{Y^L}_{\hat{h}^++\eta}|\bar{b}_k,a_k,\alpha)\leq CVaR_\alpha(\bar{R}^{Y^L}_{g}|\bar{b}_k,a_k,\alpha)$ because CVaR is a coherent risk measure. 
	$$\begin{aligned}
		A_2&=\underbrace{P(CVaR_\alpha (\bar{R}^{Y^L}_{\hat{h}^++\eta}) - CVaR_\alpha (\bar{R}^{Y^L}_{g}) > 0 |\bar{b}_k,a_k,\pi, \sup_{x\in \mathbb{R}}\{g(x)-\hat{h}^+(x)\}>\eta)}_{\leq 1} \\
		&\times \underbrace{P(\sup_{x\in \mathbb{R}}\{g(x)-\hat{h}^+(x)\}>\eta|\bar{b}_k,a_k,\pi)}_{\leq \delta/4} \\
		&+
		P(CVaR_\alpha (\bar{R}^{Y^L}_{\hat{h}^++\eta} - CVaR_\alpha (\bar{R}^{Y^L}_{g}) > 0 |\bar{b}_k,a_k,\pi, \sup_{x\in \mathbb{R}}\{g(x)-\hat{h}^+(x)\}\leq\eta) \\
		&\times \underbrace{P(\sup_{x\in \mathbb{R}}\{g(x)-\hat{h}^+(x)\}\leq\eta|\bar{b}_k,a_k,\pi)}_{\leq 1} \\
		&\leq \frac{\delta}{4} + \underbrace{P(CVaR_\alpha (\bar{R}^{Y^L}_{\hat{h}^++\eta}) - CVaR_\alpha (\bar{R}^{Y^L}_{g}) > 0 |\bar{b}_k,a_k,\pi, \sup_{x\in \mathbb{R}}\{g(x)-\hat{h}^+(x)\}\leq\eta)}_{=0}=\frac{\delta}{4}
	\end{aligned}$$
	The first inequality holds from Theorem \ref{thm:h_guarantees} for $N_\Delta \geq -ln(\frac{((\delta/4)/I)/(T-1-k)}{2})\frac{2B^2}{\eta^2 / (T-1-k)^2}$, and the last equality holds because $CVaR_\alpha(R^{Y^L}_{\hat{h}^++\eta}|\bar{b}_k,a_k,\alpha)\leq CVaR_\alpha(R^{Y^L}_{g}|\bar{b}_k,a_k,\alpha)$ as we showed above. Now we get 
	$$P(\hat{Q}_{M_{P_s}}^{\pi,h+v}(\bar{b}_k,a_k,\alpha) - Q_{M_P}^\pi(\bar{b}_k,a_k,\alpha)>v|\bar{b}_k,\pi) \leq \delta.$$
\end{proof}

\begin{theorem}\label{thm:tight_lower_bound_guarantees_theoretical_g}
	Let $\alpha>0,\delta\in (0,1)$, g and $Y^L$ as in \ref{thm:tight_lower_bound_for_v_and_q}, and $\{R^{Y^L}_i\}_{i=1}^{N_\Delta}$ a sample from $Y^L$. Then for $v=\frac{2R_{max}(T-k+1)}{\alpha}\sqrt{\frac{ln(1/\delta)}{2N_\Delta}}$ it holds that$$
	P(\hat{C}_\alpha(\{R^{Y^L}_i\}_{i=1}^{N_\Delta}) - Q^\pi_{M_P}(\bar{b}_k, a_k,\alpha)>v)\leq \delta.
	$$
\end{theorem}
\begin{proof}
	$$\begin{aligned}
		&P(\hat{CVaR}_\alpha^{M_{P_s}}(Y^L|\bar{b}_k,a_k,\pi)-Q^\pi_{M_P}(\bar{b}_k,a_k,\alpha)>v) \\
		&=P(\hat{CVaR}_\alpha^{M_{P_s}}(Y^L|\bar{b}_k,a_k,\pi)-CVaR_\alpha^{M_{P_s}}(Y^L|\bar{b}_k,a_k,\pi)+CVaR_\alpha^{M_{P_s}}(Y^L|\bar{b}_k,a_k,\pi) \\
		&-Q^\pi_{M_P}(\bar{b}_k,a_k,\alpha)>v) \\
		&\overset{1}{\leq} P(\hat{CVaR}_\alpha^{M_{P_s}}(Y^L|\bar{b}_k,a_k,\pi)-CVaR_\alpha^{M_{P_s}}(Y^L|\bar{b}_k,a_k,\pi)>v) \\
		&+P(CVaR_\alpha^{M_{P_s}}(Y^L|\bar{b}_k,a_k,\pi)-Q^\pi_{M_P}(\bar{b}_k,a_k,\alpha)>0) \\
		&\overset{2}{=}P(\hat{CVaR}_\alpha^{M_{P_s}}(Y^L|\bar{b}_k,a_k,\pi)-CVaR_\alpha^{M_{P_s}}(Y^L|\bar{b}_k,a_k,\pi)>v) \\
		&=P(CVaR_\alpha^{M_{P_s}}(Y^L|\bar{b}_k,a_k,\pi) - \hat{CVaR}_\alpha^{M_{P_s}}(Y^L|\bar{b}_k,a_k,\pi)<v) \\
		&\overset{3}{\leq} \delta
	\end{aligned}
	$$
	$^1$ Theorem \ref{thm:prob_x_plus_y_bound} \\
	$^2$ $CVaR_\alpha^{P_s}(Y^L|\bar{b}_k,a_k,\pi)$ is a deterministic lower bound for $Q^\pi_{M_P}(\bar{b}_k,a_k,\alpha)$ (Theorem \ref{thm:tight_lower_bound_for_v_and_q}). \\
	$^3$ Theorem \ref{thm:brown_bounds}.
\end{proof}

\subsection{General proofs}

\begin{theorem}\label{thm:prob_x_plus_y_bound}
	Let $\{X_i\}_{i=1}^n$ be random variables and $\epsilon_i>0$ for $i\in \{1,\dots,n\}$, then $$P(\sum_{i=1}^n X_i>\sum_{i=1}^n \epsilon_i)\leq \sum_{i=1}^n P(X_i>\epsilon_i).$$
\end{theorem}
\begin{proof}
	$$
	\begin{aligned}
		P(\sum_{i=1}^n X_i>\sum_{i=1}^n \epsilon_i)&=1-P(\sum_{i=1}^n X_i\leq \sum_{i=1}^n \epsilon_i)\leq 1 - P(\forall i\in\{1\dots,n\},X_i\leq \epsilon_i) \\
		&=P(\cup_{i=1}^n \{X_i>\epsilon_i\})\leq \sum_{i=1}^n P(X>\epsilon_i)
	\end{aligned}
	$$
\end{proof}

\subsection{Algorithm 1}
The GENPF function in the following algorithm was presented in \citep{lim2023optimality}.

\begin{algorithm}[H]
	\caption{CVaR Value function estimation for policy $\pi$}
	\label{alg:qv_estimation}
	\textbf{Global Variables:$C,N_x,L,\pi$}
	
	\begin{algorithmic}[1]
		\Procedure{GENPF}{$\overline{b}$, $a$}
		\State \textbf{Input:} particle belief set $\bar{b} = \{(x_i, w_i)\}$, action $a$.
		\State \textbf{Output:} New updated particle belief set $\bar{b'} = \{(x'_i, w'_i)\}$, mean cost $\rho$.
		\State $x_0 \gets$ sample $x_i$ from $\overline{b}$ w.p. $w_i / \sum_i w_i$
		\State $z \gets G(x_0, a)$
		\For{$i = 1,...,C$}
		\State $x'_i, r_i \gets G(x_i, a)$
		\State $w'_i \gets w_i \cdot Z(z|a, x'_i)$
		\EndFor
		\State $\bar{b'} \gets \{(x'_i, w'_i)\}_{i=1}^{N_x}$
		\State $\rho \gets \sum_i w_i r_i / \sum_i w_i$
		\State \textbf{return} $\bar{b'}, \rho$
		\EndProcedure
	\end{algorithmic}
	
	\begin{algorithmic}[1]
		\Procedure{SampleReturn}{$\bar{b}, a, t$}
		\State \textbf{Input:} Particle belief set $\bar{b} = \{(x_i, w_i)\}$, action a, depth $t$, policy $\pi$.
		\State \textbf{Output:} A scalar that is an estimate of the return with depth t.
		
		\If{t=0}
		\State \textbf{return} 0
		\EndIf
		
		\State \textbf{$\bar{b'},\rho \leftarrow GenPF(\bar{b}, a)$}
		\State \textbf{return} $\rho$ + SampleReturn$(\bar{b'}, \pi(b'),t-1)$
		\EndProcedure
	\end{algorithmic}
	
	\begin{algorithmic}[1]
		\Procedure{CVaREstimate}{$samp,\alpha$}
		\State \textbf{Input:} Sample of scalars samp and $\alpha\in [0,1]$.
		\State \textbf{Output:} A scalar that is an estimate of $CVaR$.
		\State Denote by $z_1,\dots,z_n$ the sorted (ascending order) scalars in samp and $z_0=0$.
		\State \textbf{return:} $z_n-\frac{1}{\alpha}\sum_{i=0}^{n-1} (z_{i+1}-z_i)\Bigl(\frac{i}{n}-(1-\alpha) \Bigr)^+$
		\EndProcedure
	\end{algorithmic}
	
	\begin{algorithmic}[1]
		\Procedure{Estimate$V^\pi$}{$\bar{b}, \alpha$, t}
		\State \textbf{Input:} Particle belief set $\bar{b} = \{(x_i, w_i)\}$, depth $t$, policy $\pi$.
		\State \textbf{Output:} A scalar $\hat{V}^\pi(\bar{b})$ that is an estimate of $V^\pi(b)$.
		\State \textbf{return} EstimateQ$(\bar{b}, \pi(\bar{b}), t, \alpha)$
		\EndProcedure
	\end{algorithmic}
	
	\begin{algorithmic}[1]
		\Procedure{EstimateQ}{$\bar{b}, a, t, \alpha$}
		\State \textbf{Input:} Particle belief set $\bar{b} = \{(x_i, w_i)\}$, action $a$, depth $t$, policy $\pi$.
		\State \textbf{Output:} A scalar $\hat{Q}^{\pi}(\bar{b}, a)$ that is an estimate of $Q^\pi(b, a)$.
		\For{$i = 1, \dots, C$}
		\State $R_i \gets \text{SampleReturn}(\bar{b}, a, t)$
		\EndFor
		\State $samp \gets R_1,\dots, R_C$
		\State \textbf{return} CVaREstimate$(samp,\alpha)$
		\EndProcedure
	\end{algorithmic}
\end{algorithm}
\end{document}